\documentclass[graybox]{svmult}

\usepackage{mathptmx}       
\usepackage{helvet}         
\usepackage{courier}        
\usepackage{graphicx}        
\usepackage[bottom]{footmisc}
\usepackage{amsmath}
\usepackage{url} %

\begin{document}

\title*{Non-crossing Dependencies: Least Effort, not Grammar}
\author{Ramon Ferrer-i-Cancho}
\institute{Ramon Ferrer-i-Cancho \at Complexity and Quantitative Linguistics Lab, LARCA Research Group. 
Departament de Ci\`encies de la Computaci\'o, Universitat Polit{\`e}cnica de Catalunya (UPC).  Campus Nord, Edifici Omega, Jordi Girona Salgado 1-3. 08034 Barcelona, Catalonia (Spain), \email{rferrericancho@cs.upc.edu}.}
%
%
\maketitle

\abstract{The use of null hypotheses (in a statistical sense) is common in hard sciences but not in theoretical linguistics.
Here the null hypothesis that the low frequency of syntactic dependency crossings is expected by an arbitrary ordering of words is rejected. It is shown that this would require star dependency structures, which are both unrealistic and too restrictive. The hypothesis of the limited resources of the human brain
is revisited. Stronger null hypotheses taking into account actual dependency lengths for the likelihood of crossings are presented. Those hypotheses suggests that crossings are likely to reduce when dependencies are shortened. 
A hypothesis based on pressure to reduce dependency lengths is more parsimonious than a principle of minimization of crossings or a grammatical ban that is totally dissociated from the general and non-linguistic principle of economy.
}

\section{Introduction}

\begin{figure}
\includegraphics[scale = 0.75]{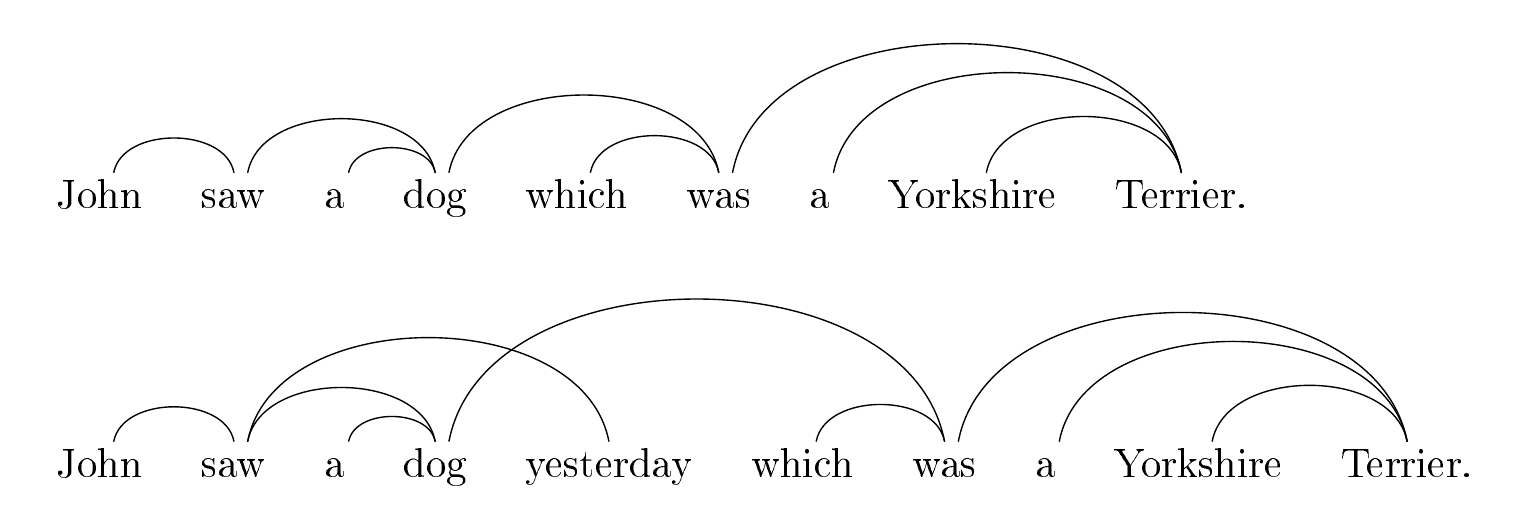}
\caption{\label{real_sentences_figure} {\bf Top} an English sentence without crossings. {\bf Bottom} a variant of the previous sentence with one dependency crossing (the dependency between "saw" and "yesterday" crosses the dependency between "dog" and "was" and vice versa). Adapted from \cite{McDonald2005a}. }
\end{figure}


A substantial subset of theoretical frameworks under the general umbrella of "generative grammar" or "generative linguistics" have been kidnapped by the idea that a deep theory of syntax requires that one neglects the statistical properties of the system \cite{Miller1963} and abstracts away from functional factors such as the limited resources of the brain \cite{Chomsky1965}. 

This radical assumption disguised as intelligent abstraction led to the distinction between competence and performance (see \cite[Sect. 2.4]{Jackendoff2002a} for a historical perspective from generative grammar), a dichotomy that is sometimes regarded as a soft methodological division \cite[pp. 34]{Jackendoff2002a} or as theoretically unmotivated \cite{Newmeyer2001a}. A sister radical dichotomy is the division between grammar and usage \cite{Newmeyer2003a}. A revision of those views has led to the  proposal of competence-plus, \emph{"a package consisting of the familiar recursive statements of grammatical rules,
plus a set of numerical constraints on the products of those rules"} \cite{Hurford2012_Chapter3}.
Interestingly, certain approaches reconcile competence with performance by regarding  grammar as a store of "frozen" or "fixed" performance preferences \cite[p. 3]{Hawkins2004a} or by opening the set of numerical constraints of competence-plus to performance factors \cite{Hurford2012_Chapter3}.
Other examples of approaches that reject the dichotomous view of language are emergent grammar \cite{Hopper1998a}, synergetic linguistics \cite{Koehler2005a} or probabilistic syntax \cite{Manning2002a}.
The challenges of the competence/performance are not specific to generative linguistics. For instance, \emph{"the competence/performance distinction is also embodied in many symbolic
models of language processing"} \cite{Christiansen1999a} and integrated with some refinements in language evolution research \cite{Hurford2012_Chapter3}.  

Again from the perspective of standard model selection \cite{Burnham2002a}, the competence/\-performance dichotomy, even in soft versions, has a serious risk: if a more parsimonious theory exists based on performance, one that has the same or even superior explanatory power, it  may not be discovered and if so, it will not be sufficiently endorsed. Astonishingly, linguistic theories that belittle the role of the limited resources of the human brain for structural constraints of syntax are presented as minimalistic (e.g., \cite{Chomsky1995}). In contrast, standard model selection favors theories with a good compromise between simplicity (often coming from a suitable abstraction or idealization) and explanatory power \cite{Burnham2002a}.

A follower of the competence-performance split may consider that the opponents are unable to think in sufficiently abstract terms: opponents are being side-tracked by actual language use and limited computational resources and do not focus on the essence of syntax (in those views the essence of syntax is grammar \cite{Newmeyer2003a} or certain features of such as recursion \cite{Hauser2002}; in other approaches the essence of syntax is not grammar but dependencies \cite{Hudson2007a,Frank2012a}). 
The proponents of the split doctrine have not hesitated either to advertise functional approaches to linguistic theory as wrong \cite{Miller1968a} or to attempt to dismantle attempts to turn research on language or communication more quantitative (e.g., \cite{Miller1963,Niyogi1995a,Suzuki2004a}). A real scientist however, will ask for the quality of a theory or a hypothesis in terms of the accuracy of its definitions, its testability, the statistical analyzes that have been performed to support it, the null hypotheses, the trade-off between explanatory power and parsimony of the theory, and so on. 

If the limited resources of the brain are denied, one might be forced to blame grammar for the occurrence of certain patterns. Using standard model selection terms \cite{Burnham2002a}, forwarding the responsibility to grammar implies the addition of more parameters to the model, indeed unnecessary parameters, as it will be shown here through a concrete phenomenon. The focus of the current article is a striking pattern of syntactic dependency trees of sentences that was reported in the 1960s: dependencies between words normally do not cross when drawn over the sentence \cite{Lecerf1960a,Hays1964} (e.g., Fig. \ref{real_sentences_figure}).
The problem of dependency crossings looks purely linguistic but it goes beyond human language: crossings have also been investigated in dependency networks where vertices are occurrences of nucleotides $A$, $G$, $U$, and $C$ and edges are $U$-$G$ and Watson-Crick base pairs, i.e. $A$-$U$, $G$-$C$ \cite{Chen2009a}. Having in mind various domains of application helps a researcher to apply the right level of abstraction. Becoming a specialist in human language or certain linguistic phenomena helps to find locally optimal theories, causes the illusion of scientific success when becoming the world expert of a certain topic but does not necessarily produce compact, coherent, general and elegant theories.

Here new light is shed on the origins of non-crossing dependencies by means of two fundamental tools of the scientific method: null hypotheses and unrestricted parsimony (unrestricted parsimony in the sense of being a priori open to favor theories that make fewer assumptions; not in the sense that  parsimony has to be favored neglecting explanatory power). Unfortunately, the definition of null hypotheses (in a statistical sense) is rare in theoretical linguistics (although it is fundamental in biology or medicine). Even in the context of quantitative linguistics research, clearly defined null hypotheses or baselines are present in certain investigations, e.g., \cite{Ferrer2012d,Ferrer2004b} but are missing in others e.g., \cite{Ferrer2011c, Moscoso2013a}. When present, they are not always tested rigorously \cite{Ferrer2009b}. In the context of quantitative research, claims about the efficiency of language have been made lacking a measure of cost and evidence that such a cost is below chance \cite{Piantadosi2011a}. A deep theory of language requires (at least) metrics of efficiency, tests of their significance and an understanding of the relationship between the minimization of the costs that they define and the emergence of the target patterns, e.g., Zipf's law of abbreviation \cite{Ferrer2012d}.

To our knowledge, claims for the existence of a universal grammar have never been defended by means of a null hypothesis (in a statistical sense), e.g., \cite{Jackendoff2002a,Uriagereka1998},
and a baseline is missing in research where grammar is seen as a conventionalization of performance constraints \cite{Hawkins2004a} or in research where competence is complemented with quantitative constraints \cite{Hurford2012_Chapter3}.  As for the latter, baselines would help one to determine which of those constraints must be stored by grammar or competence-plus.

The first question that a syntactician should ask as a scientist when investigating the origins of a syntactic property $X$ is: could $X$ happen by chance? The question is equivalent to asking if grammar (in the sense of some extra knowledge) or specific genes are  needed to explain property $X$.
Accordingly, the major question that this article aims to answer is: could the low frequency of crossings in syntactic dependency trees be explained by chance, maybe involving general computational constraints of the human brain?

The remainder of the article is organized as follows. Sect. \ref{dependency_structure_section} reviews our minimalistic approach to the syntactic dependency structure of sentences. Sect. \ref{null_hypothesis_section} considers the null hypothesis of a random ordering of the words of the sentence and shows that keeping the expected number of crossings small requires unrealistic constraints on the ensemble of possible dependency trees (only star trees would be possible). Sect. \ref{alternative_hypotheses_section} considers alternative hypotheses, discarding the vague or heavy hypothesis of grammar and focusing on two major hypotheses: a principle of minimization of crossings and a principle of minimization of the sum of dependency lengths. The analysis suggests that the number of crossing and the sum of dependency lengths are not perfectly correlated but their correlation is strong. Of the two principles, dependency length minimization offers a more parsimonious account of many more linguistic phenomena. Interestingly, that principle is motivated by the need of minimizing cognitive effort.
A challenge for the hypothesis that the rather small number of crossings of real sentence is a side-effect of minimization of dependency lengths is (a) determining the degree of that minimization that the real number of crossings requires and (b) if that degree is realistic.     
Sect. \ref{stronger_null_hypothesis_section} presents a stronger null hypothesis that addresses the challenge with knowledge about edge lengths. That null hypothesis allows one to predict the number of crossings when the length of one of the edges potentially involved in a crossing is known but words are arranged at random. Thus, that predictor uses the actual dependency lengths to estimate the number of crossings. Interestingly, that predictor provides further support for a strong correlation between crossings and dependency lengths: analytical arguments suggest that it is likely that a reduction of dependency lengths causes a drop in the number of crossings. Sect. \ref{another_stronger_null_hypothesis_section} considers another predictor based on a stronger null hypothesis where the sum of dependency lengths is given but words are arranged at random. Preliminary numerical results indicate a strong correlation between the mean number of crossings and the sum of dependency lengths over all the possible orderings of the words of real sentences.  
Interestingly, that null hypothesis leads to a predictor that requires less information about a real sentence than the previous predictor (only the sum of dependency lengths is needed) and paves the way to understanding the rather low number of crossings in real sentences as a consequence of global cognitive constraints on dependency lengths. Sect. \ref{prediction_and_testing_section} compares the predictions of the three predictors on a small set of sentences. The results suggest that the predictor based on the sum of dependency lengths is the best candidate. There it is also demonstrated that p-value testing can be used to investigate the adequacy of the best candidate. Interestingly, the best candidate was not rejected in that sample of sentences. Finally, Sect. \ref{discussion_section} reviews and discusses the idea that least effort, not grammar, is the reason for the small number of crossings of real sentences. 

\section{The Syntactic Dependency Structure of Sentences}

\label{dependency_structure_section}

This article borrows the minimalistic approach to the syntactic dependency structure of sentences of dependency grammar \cite{Hudson1984,Melcuk1988} and recent progress in cognitive sciences \cite{Frank2012a}: 
\begin{itemize}
\item
No hierarchical phrase structure is assumed in the sense that the structure of a sentence is defined simply as a tree where vertices are words and edges are syntactic dependencies. This is a fundamental assumption of our approach: tentatively, the network defining the dependencies between words might be a disconnected forest or a graph with cycles (these are possibilities that have not been sufficiently investigated) \cite{Hudson1984}. A general theory of crossings in nature cannot obviate the fact that RNA structures cannot be modeled with trees but can be modeled with forests \cite{Chen2009a} \footnote{In those RNA structures, vertex degrees do not exceed one \cite{Chen2009a} and thus cycles are not possible but connectedness is not either (the handshaking lemma \cite[p. 4]{Bollobas1998} indicates that such a graph cannot have more than $n/2$ edges, being $n$ the number of vertices, and thus cannot be connected because that needs at least $n-1$ edges). }.
Although the choice of a tree of words as the reference model for sentence structure (e.g., \cite{McDonald2005a}) is to some extent arbitrary, a tree is optimal for being the kind of network that is able to connect all words with the smallest amount of edges \cite{Ferrer2003a}. 
\item
Words establish direct relationships that are not necessarily mediated by syntactic categories (non-terminals in the phrase structure formalism and generative grammar evolutions). This skepticism about syntactic categories (as entities by its own, not epiphenomena) goes beyond dependency grammar, e.g., construction grammar \cite{Goldberg2003a}.
\end{itemize}
Along the lines of \cite{Frank2012a}, link direction is irrelevant for the arguments in this article. Even within the dependency grammar formalism, dependencies are believed to be directed (from heads to modifiers/complements) \cite{Melcuk1988,Hudson1984}. A minimalistic approach to dependency syntax should not obviate the fact that the accuracy of dependency parsing improves if link direction is neglected \cite{GomezRodriguez2014a}. 

\section{The Null Hypothesis}

\label{null_hypothesis_section}

Let $n$ be the number of vertices of a tree. 
Let $k_i$ be the degree of the $i$-th vertex of a tree and $k_1,...,k_i,...,k_n$ its degree. By $K_\alpha$, we denote  
\begin{equation}
K_\alpha = \sum_{i=1}^n k_i^\alpha,
\end{equation}
where $\alpha$ is a natural number.
In a tree, $K_1$ only depends on $n$, i.e. \cite{Noy1998a}, 
\begin{equation}
K_1 = 2(n-1)
\label{sum_of_degrees_equation}
\end{equation}
and thus the 1st moment of degree is 
\begin{equation}
\left<k \right> = \frac{K_1}{n} = 2 - \frac{2}{n}.
\label{degree_1st_moment_equation}
\end{equation}

Let $E_0[C]$ be the expected number of crossings in a random linear arrangement of a dependency tree with a given degree sequence, i.e. \cite{Ferrer2013d}
\begin{equation}
E_0[C] = \frac{n}{6}\left(n-1-\left< k^2\right>\right),
\label{expected_crossings_equation}
\end{equation} 
where $\left< k^2\right>$ is the 2nd moment of degree, i.e. 
\begin{equation}
\left< k^2\right> = \frac{K_2}{n}.
\label{degree_2nd_moment_equation}
\end{equation}
Thus, the expected number of crossings depends on the number of vertices ($n$) and the 2nd moment of degree ($\left< k^2 \right>$). The higher the hubiness (the higher $\left< k^2 \right>$) the lower the expected number of crossings.

\begin{figure}
\sidecaption[t]
\includegraphics[scale = 0.75]{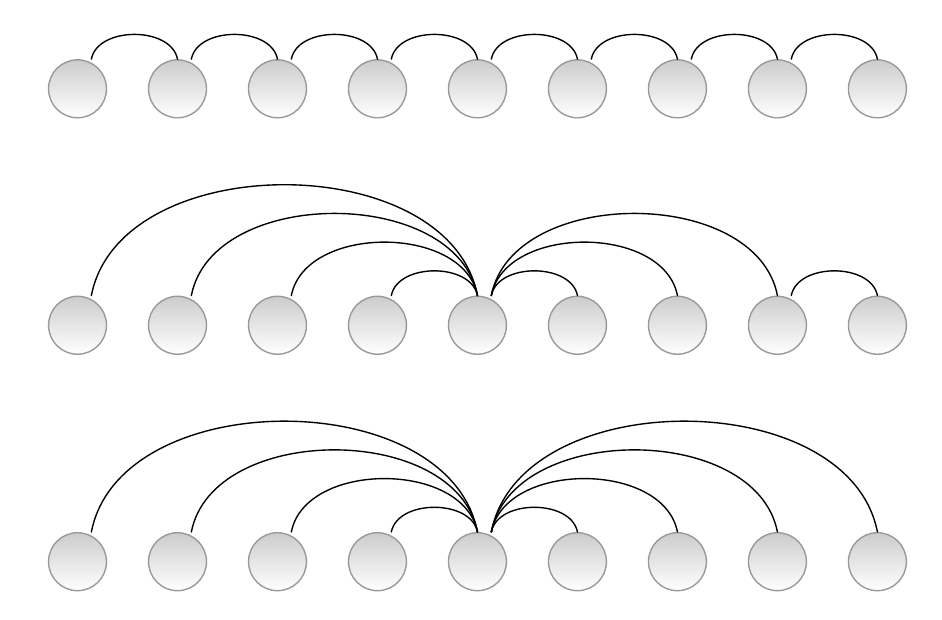}
\caption{\label{kinds_of_trees_figure} Linear arrangements of trees of nine vertices. {\bf Top} a linear tree.
{\bf Center} quasi-star tree. {\bf Bottom} star tree. }
\end{figure}

A star tree is a tree with a vertex of degree $n-1$ while a linear tree is a tree where no vertex degree exceeds two \cite{Ferrer2013d} (Fig. \ref{kinds_of_trees_figure}). For a given number of vertices, $E_0[C]$ is minimized by star trees, for which $E_0[C]=0$, whereas $E_0[C]$ is maximized by linear trees, for which \cite{Ferrer2013b,Ferrer2013d}
\begin{equation}
E_0[C] = \frac{1}{6}n(n - 5) + 1.
\end{equation}

As $E_0[C]$ depends on the $\left< k^2 \right>$ of the tree, the null hypothesis that the tree structures are chosen uniformly at random among all possible labeled trees is considered next.   
The Aldous-Brother algorithm allows one to generate uniformly random labeled spanning trees from a graph \cite{Aldous1990a,Broder1989a}. Here a complete graph is assumed to be the source for the spanning trees. 
A low number of crossings cannot be attributed to grammar if $E_0[C]$ is low.

$E_0[C]$ is the expectation of $C$ given a degree sequence. Indeed, that expectation can be obtained just from knowledge about $\left< k^2 \right>$ and $n$ (Eq.~\ref{expected_crossings_equation}).
The expectation of $C$ for uniformly random labeled trees is
\begin{proposition}
\begin{align}
E[E_0[C]] & = \frac{1}{6} (n-1)\left( n-5+\frac{6}{n} \right) \nonumber \\
          & = \frac{n^2}{6} - n + \frac{11}{6} - \frac{1}{n}.
\end{align}
\end{proposition}
\begin{proof}
On the one hand, the degree variance for uniformly random labeled trees is \cite{Moon1970a,Noy1998a} 
\begin{equation}
V[k] = \left<k^2 \right> - \left<k \right>^2 =  \left(1- \frac{1}{n}\right)\left(1-\frac{2}{n}\right).
\end{equation}
Applying Eq.~\ref{degree_1st_moment_equation}, it is obtained
\begin{equation}
\left<k^2 \right> = \left(1-\frac{1}{n}\right)\left(5 - \frac{6}{n}\right).
\label{degree_2nd_moment_null_hypothesis_equation}
\end{equation}
On the other hand, 
\begin{align}
E[E_0[C]] & = E\left[ \frac{n}{6}\left( n - 1 - \left< k^2 \right> \right) \right] \nonumber && \text{applying Eq.~\ref{expected_crossings_equation}} \\
        & = \frac{n}{6} \left( n-1-E \left[ \left< k^2 \right> \right] \right) \nonumber \\
        & = \frac{n}{6} \left( n-1- \left(1-\frac{1}{n}\right)\left(5 - \frac{6}{n}\right) \right) && \text{applying Eq.~\ref{degree_2nd_moment_null_hypothesis_equation}} \nonumber \\
        & = \frac{1}{6} (n-1)\left( n-5+\frac{6}{n} \right). \nonumber 
\end{align}
\qed
\end{proof}

Fig. \ref{expected_crossings_figure} shows that uniformly random labeled trees exhibit a high value of $E_0[C]$ that is near the upper bound defined by linear trees. 
Thus, it is unlikely that the rather low frequency of crossings in real syntactic dependency trees \cite{Melcuk1988,Liu2010a} is due to uniform sampling of the space of labeled trees. However, one cannot exclude the possibility that real dependency trees belong to a subclass of random trees for which $E_0[C]$ is low (e.g., the uniformly random trees may not be spanning trees of a complete graph). This possibility is explored next.

\begin{figure}
\includegraphics[scale = 0.5]{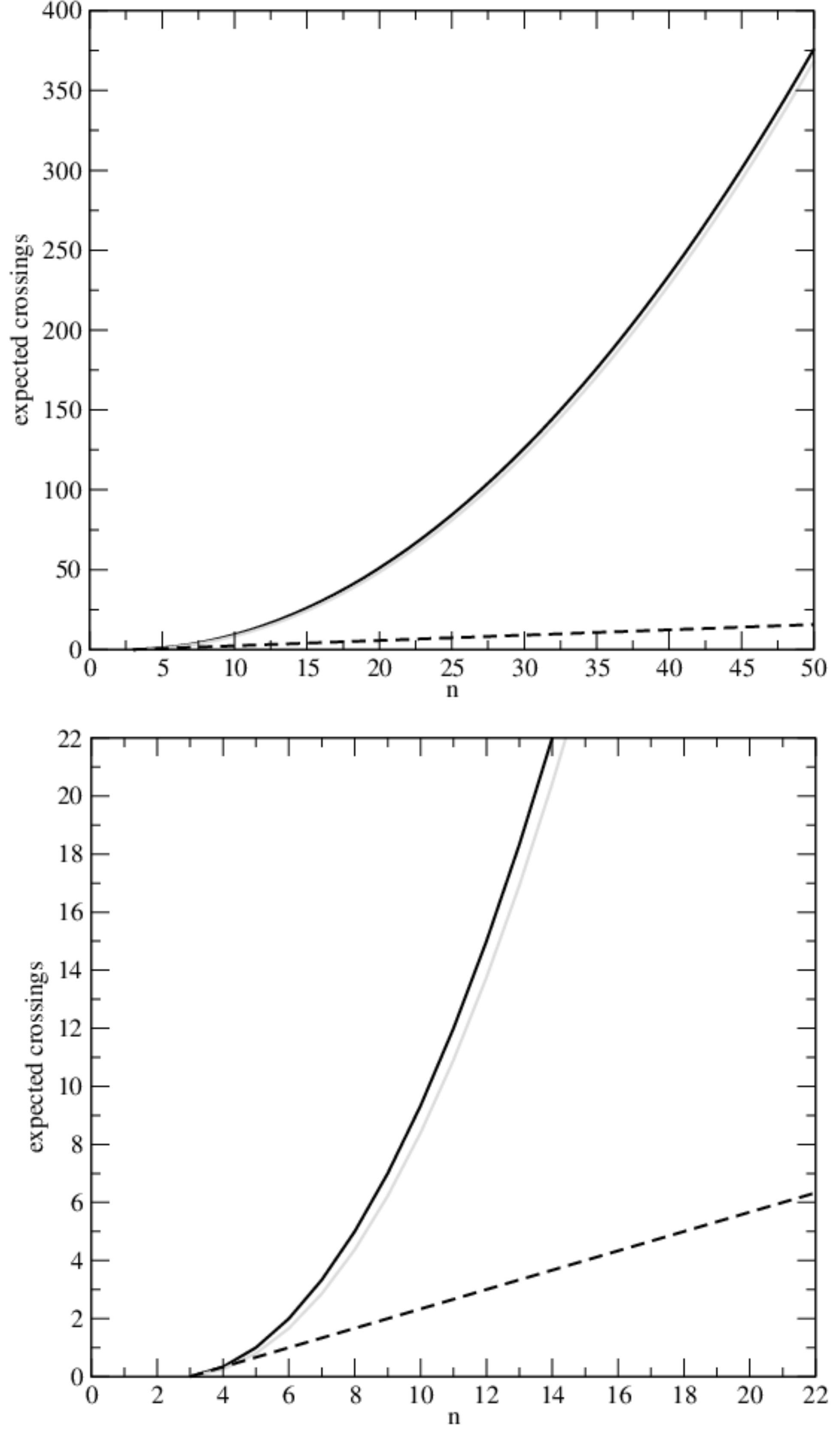}
\caption{
\label{expected_crossings_figure}
The expected number of crossings as a function of the number of vertices of the tree ($n$) in random linear arrangements of vertices for linear trees (\emph{black solid line}), uniformly random labeled trees (\emph{gray line}) and quasi-star trees (\emph{dashed line}). {\bf Top} the whole picture up to $n = 50$. {\bf Bottom} a zoom of the left-bottom corner.} 
\end{figure}

A quasi-star tree is defined as a tree with one vertex of degree $n - 2$, one vertex of degree $2$ and the remainder of vertices of degree $1$ (Fig. \ref{kinds_of_trees_figure}). A quasi-star tree needs $n \geq 3$ to exist (Appendix). 
The sum of squared degrees of such a tree is (Appendix)
\begin{equation}
K_2^\text{quasi} = n^2 - 3n + 6
\end{equation}
and thus the degree 2nd moment of a quasi-star tree is 
\begin{equation}
\left< k^2 \right>^\text{quasi} = \frac{K_2^\text{quasi}}{n} = n - 3 + \frac{6}{n}.
\label{2nd_moment_of_quasi_star_tree_equation}
\end{equation}
Eq.~\ref{expected_crossings_equation} and Eq.~\ref{2nd_moment_of_quasi_star_tree_equation} allow one to infer that 
\begin{equation}
E_0[C] = \frac{n}{3} - 1
\end{equation}
for a quasi-star tree. Fig. \ref{expected_crossings_figure} shows the linear growth of the expected number of crossings as a function of the number of vertices in quasi-star trees.
Interestingly, if a tree has a value of $\left< k^2 \right>$ that exceeds $\left< k^2 \right>^\text{quasi}$, it has to be a star tree (see Appendix). For this reason, Fig. \ref{expected_crossings_figure} suggests that star trees are the only option to obtain a small constant number of crossings. A detailed mathematical argument will be presented next.

If it is required that the expected number of crossings does not exceed $a$, i.e. $E_0[C]\leq a$, Eq.~\ref{expected_crossings_equation} gives
\begin{equation}
\left< k^2\right> \geq n - 1 - \frac{6a}{n}.
\label{minimum_2nd_moment_equation}
\end{equation} 
Notice that the preceding result has been derived making no assumption about the tree topology.
We aim to investigate when a $E_0[C] \leq a$ implies a star tree. 


As a tree whose value of $\left< k^2 \right>$ exceeds $\left< k^2 \right>^\text{quasi}$ must be a star tree (Appendix), Eq.~\ref{minimum_2nd_moment_equation} indicates that if
\begin{equation}
n - 1 - \frac{6a}{n} > \left< k^2\right>^\text{quasi}
\label{intermediate_equation}
\end{equation}
then $E_0[C] \leq a$ requires a star tree. Applying the definition of $\left< k^2\right>^\text{quasi}$ in Eq.~\ref{2nd_moment_of_quasi_star_tree_equation} to Eq.~\ref{intermediate_equation}, we obtain that a star tree is needed to expect at most $a$ crossings if 
\begin{equation}
n > 3a + 3.
\label{above_quasi_star_tree_equation}
\end{equation}
Thus, Eq.~\ref{above_quasi_star_tree_equation} implies that a hub tree is needed to expect at most one crossings by chance ($a=1$) for $n > 6$ (this can be checked with the help of Fig. \ref{expected_crossings_figure}).
In order to have at most one crossing by chance, the structural diversity must be minimum because star trees are the only possible labeled trees. To understand the heavy constraints imposed by $a=1$ on the possible trees, consider $t(n)$, the number of unlabeled trees of $n$ that can be formed (Table \ref{number_of_unlabelled_trees_table}). When $n=4$, the only trees than can be formed are a star tree and a linear tree, which gives $t(4) = 2$. In contrast, the star tree is only one out of 19320 possible unlabeled trees when $n=16$. The decrease in diversity is more radical as $n$ increases (Table \ref{number_of_unlabelled_trees_table}).
The choice of $a=2$ does not change the scenario so much: Eq.~\ref{above_quasi_star_tree_equation} predicts that sentences of length $n > 9$ should have a star tree structure if no more than two crossings are to be expected. Real syntactic dependency trees from sufficiently longer sentences are far from star trees, e.g., Fig. \ref{real_sentences_figure} \cite{Melcuk1988,Hays1964,Lecerf1960a}. 

\begin{table}
\caption{\label{number_of_unlabelled_trees_table} $t(n)$, the number of unlabeled trees of $n$ vertices \cite{OEIS_A000055}. 
} 
\begin{tabular}{p{0.7cm}p{1cm}}
\hline\noalign{\smallskip}
$n$ & $t(n)$ \\ 
\noalign{\smallskip}\svhline\noalign{\smallskip}
1 & 1 \\
2 & 1 \\ 
3 & 1 \\
4 & 2 \\
5 & 3 \\
6 & 6 \\
7 & 11 \\
8 & 23 \\
9 & 47 \\
10 & 106 \\
11 & 235 \\
12 & 551 \\
13 & 1301 \\
14 & 3159 \\
15 & 7741 \\ 
16 & 19320 \\
17 & 48629 \\
18 & 123867 \\
19 & 317955 \\
\noalign{\smallskip}\hline\noalign{\smallskip}
\end{tabular}
\end{table}

\section{Alternative Hypotheses}

\label{alternative_hypotheses_section}

It has been  shown that the low frequency of crossings is unexpected by chance in random linear arrangements of real syntactic dependency trees. As scientists, the next step is exploring the implications of this test and evaluating alternative hypotheses. 
Vertex degrees ($\left< k^2 \right>$), which are an aspect of sentence structure, have been discarded as the only origin for the low frequency of crossings. This is relevant for some views where 
competence or grammar concern the structure of a sentence \cite{Chomsky1965,Jackendoff2002a,Newmeyer2003a}.
Discussing what competence or grammar is or should be is beyond the scope of this article but it is worth examining common reactions of language researchers when encountering a pattern:
\begin{itemize}
\item
For statistical patterns such as Zipf's law for word frequencies and Menzerath's law, it was concluded that the patterns are inevitable \cite{Miller1963,Sole2010a} (see \cite{Ferrer2009b,Ferrer2012f} for a review of the weaknesses of such conclusions).
\item
Concerning syntactic regularities in general, a naive but widely adopted approach is blaming (universal) grammar, the faculty of language or similar concepts \cite{Newmeyer2003a,Hauser2002}. The fact that one is unable to explain a certain phenomenon through usage is considered as a justification for grammar (e.g., \cite{Newmeyer2003a}). 
However, a rigorous justification requires a proof of the impossibility of usage to account for the phenomenon. To our knowledge, that proof is never provided.
\item
In the context of dependency grammar, crossings dependencies have been banned \cite{Hudson1984} or it has been argued that most phrases cannot have crossings or that crossings turn sentences ungrammatical \cite[pp. 130]{Hudson2007a}. It is worrying that the statement is not the conclusion of a proof of the impossibility of a functional explanation. Furthermore, the argument of "ungrammaticality" is circular and sweeps processing difficulties under the carpet.
\end{itemize}
In the traditional view of grammar or the faculty of language, the limited resources of the human brain are secondary or anecdotal \cite{Hauser2002, Hudson2007a}. 
Recurring to grammar or a language faculty implies more assumptions, e.g. grammar would be the only reason why dependencies do not cross so often, and an explanation about the origins of the property would be left open (the explanation would be potentially incomplete). The property might have originated in grammar as a kind of inevitable logical or mathematical property, or might be supported by genetic information of our species, or it might also have been transferred to grammar (culturally or genetically) and so on. 
Thus, a grammar that is responsible for non-crossing dependencies would not be truly minimalistic (parsimonious) at all if the phenomenon could be explained by a universal principle of economy (universal in the sense of concerning the human brain not necessarily exclusively). This is likely the case of current approaches to dependency grammar at least (e.g. \cite{Hudson2007a,Melcuk1988,Hudson1984}.

\begin{figure}
\sidecaption[t]
\includegraphics[scale = 0.75]{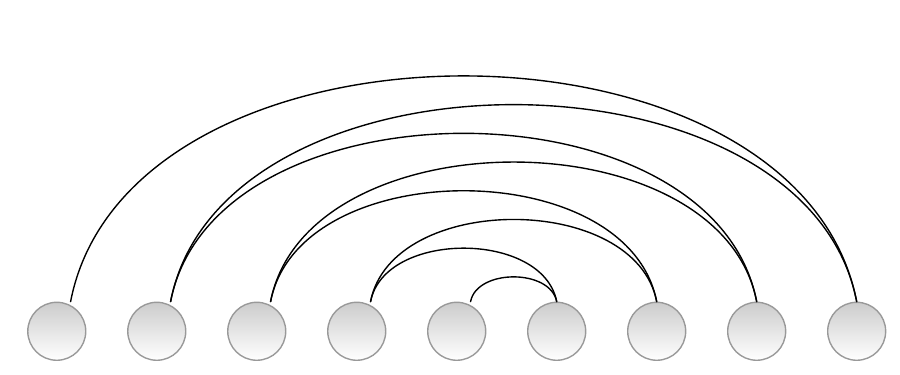}
\caption{
\label{maximal_non_crossing_linear_tree_figure}
A linear arrangements of the vertices of a linear tree that maximizes $D$ (the sum of dependency lengths) when edge crossings are not allowed. }
\end{figure}

\subsection{A Principle of Minimization of Dependency Crossings}

A tempting hypothesis is a principle of minimization of dependency crossings (e.g., \cite{Liu2008a}) which can be seen as a quantitative implementation of the ban of crossings \cite{Hudson1984,Hudson2007a}. 
This minimization can be understood as a purely formal principle (a principle of grammar detached from performance constraints but then problematic for the reasons explained above) or a principle related to performance.  
A principle of the minimization of crossings (or similar ones) is potentially problematic for at least three reasons:
\begin{itemize}
\item
It is an inductive solution that may overfit the data.
\item
It is naive and superficial because it does not distinguish the consequences (e.g., uncrossing dependencies) from the causes. A deep theory of language requires distinguishing underlying principles from products: the principle of compression \cite{Ferrer2012d} from the law of abbreviation (a product), the principle of mutual information maximization from vocabulary learning biases \cite{Ferrer2013g} (another product), and so on.
\item
If patterns are directly translated into principles, the risk is that of constructing a fat theory of language when merging the tentative theories from every domain. When integrating the principle of minimization of crossings with a general theory of syntax, one may get two principles: a principle of minimization of crossings and a principle of dependency length minimization. In contrast, a theory where the minimization of crossings is seen as a side-effect of a principle of dependency length minimization 
\cite{Ferrer2006d,Liu2008a,Morrill2009a,Ferrer2013b} might solve the problem in one shot through a single principle of dependency length minimization. However, it has cleverly been argued that a principle of minimization of crossings might imply a principle of dependency length minimization \cite{Liu2008a} and thus a principle of minimization of crossings might not imply any redundancy. 
\end{itemize}
Thus, it is important to review the hypothesis of the minimization of dependency length and the 
logical and statistical relationship with the minimization of $C$. 

\subsection{A Principle of Minimization of Dependency Lengths}

The length of a dependency is usually defined as the absolute difference between the positions involved (the 1st word of the sentence has position 1, the 2nd has position 2 and so on). In Fig. \ref{real_sentences_figure}, the length of the dependency between "John" and "saw" is $2-1= 1$ and the length of the dependency between "saw" and "dog" is $4-2 = 2$. In this definition, the units of length are word tokens (it might be more precise if defined in phonemes, for instance). 
If $d_i$ is the length of the $i$-th dependency of a tree (there are $n-1$ dependencies in a tree) and $g(d)$ is the cost of a dependency of length $d$, the total cost of a linguistic sequence from the perspective of dependency length is the total sum of dependency costs \cite{Ferrer2014a,Ferrer2013e}, which can defined as
\begin{equation}
D = \sum_{i=1}^n g(d_i),
\end{equation}
where $g(d)$ is assumed to be a strictly monotonically increasing function of $d$ \cite{Ferrer2013e}.
The mean cost of a tree structure is defined as $\left< d \right> = D/(n-1)$. If $g$ is the identity function ($g(d) = d$) then $D$ is the sum of dependency lengths (and $\left< d \right>$ is the mean dependency length).
It has been hypothesized that $D$ or equivalently $\left< d \right>$ is minimized by sentences (see \cite{Ferrer2013e} for a review).
The hypothesis does not imply that the actual value of $D$ has to be the minimum in absolute terms. Hereafter we assume that $g(d)$ is the identity function ($g(d) = d$). 

The minimum $D$ that can be obtained without altering the structure of the tree is the solution of the minimum linear arrangement problem in computer science \cite{Chung1984}. Another baseline is provided by  
the expected value of $D$ in a random arrangement of the vertices, which is \cite{Ferrer2004b}
\begin{equation}
E_0[D]=(n-1)(n+1)/3.
\label{expected_length_equation}
\end{equation}
Statistical analyzes of $D$ in real syntactic dependency trees have revealed that $D$ is systematically below chance (below $E_0[D]$) for sufficiently long sentences  
but above the value of a minimum linear arrangement on average \cite{Ferrer2004b,Ferrer2013c}. 

$D$ is one example of a metric or score to evaluate the efficiency of a sentence from a certain dimension (see \cite{Morrill2000a,Hawkins1998a} for similar metrics on syntactic structures). Stating clearly the metric that is being optimized is a requirement for a rigorous claim about efficiency of language. 
For instance, consider the sentence on top of Fig. \ref{extraposition_figure} and the version below that arises from the right-extraposition of the clause "who I knew". Notice that the dependency tree is the same in both cases (only word order varies). It has been argued that 
theories of processing based on the distance between dependents \emph{
"predict that an extraposed relative clause would be more difficult to process than an in situ, adjacent relative clause"} \cite{Levy2012a}. However, that does not grant one to conclude that the sentence on top of Fig. \ref{extraposition_figure} should be easier to process than the sentence below from that perspective: one has  
 $D = 3 \times 1 + 2 + 4 + 6 = 15$ for the sentence without the extraposition and $D = 3 \times 1 + 2 \times 2 + 3 = 10$ for the one with right-extraposition suggesting that the easier sentence is precisely the sentence with right-extraposition. 
The prediction about the cost of extraposition in \cite{Levy2012a} is an incomplete argument. The ultimate conclusion about the cost of extraposition requires considering all the dependency lengths, i.e. a true efficiency score. A score of sentence locality is needed to not rule out prematurely accounts of the processing difficulty of non-projective orderings that are based purely on \emph{"dependency locality in terms of linear positioning"} \cite{Levy2012a}. The issue is tricky even for studies where a quantitative metric of dependency length such as $D$ is employed: it is important to not mix values of the metric coming from sentences of different length to draw solid conclusions about a corpus or a language \cite{Ferrer2013c}.  
The need of strengthening quantitative standards \cite{Gibson2010a} and also the need of appropriate controls \cite{Culicover2010a,Ferrer2011c,Moscoso2013a} in linguistic research are challenges that require the serious commitment of each of us. 

\begin{figure}
\includegraphics[scale = 0.75]{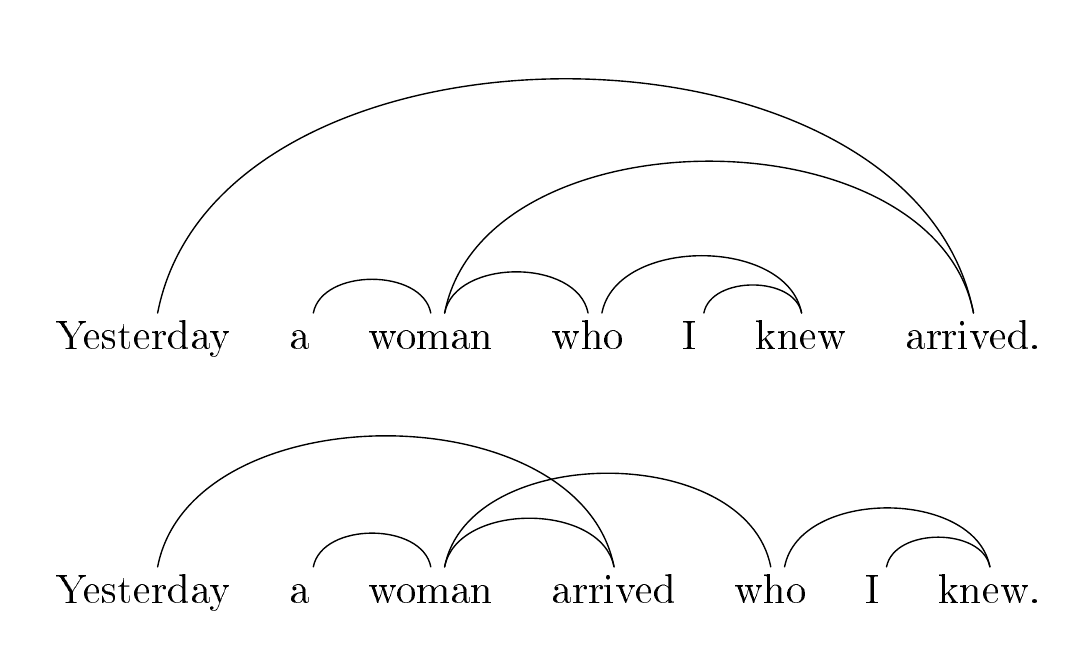}
\caption{\label{extraposition_figure} {\bf Top} an English sentence with a relative clause ("who I knew"). {\bf Bottom} the same sentence with a right extraposition of the relative clause. Adapted from \cite{Levy2012a}. }
\end{figure}

\subsection{The Relationship Between Minimization of Crossings and Minimization of Dependency Lengths}
 
Let us examine the logical relationship between the two principles above from the perspective of their global minima. 
On the one hand, the minimum value of $C$ is 0 \cite{Ferrer2013b} and the minimum value of $D$ is obtained by solving the minimum linear arrangement problem \cite{Baronchelli2013a} or a generalization \cite{Ferrer2013e}, which yields $D_\text{min}$. At constant $n$ and $\left< k^2\right>$, there are two facts:
\begin{itemize}
\item
$C=0$ does not imply $D = D_\text{min}$ in general. This can be shown by means of two extreme configurations, a star tree, which maximizes $\left< k^2 \right>$ and a linear tree, which minimizes $\left< k^2 \right>$ \cite{Ferrer2013d}:
   \begin{itemize}
   \item
   A star tree implies $C = 0$ \cite{Ferrer2013b}. In that tree, 
$D=D_\text{min}$ holds only when the hub is placed at the center \cite{Ferrer2013e}. If the hub is placed at one of the extremes of the sequence, $D$ is maximized for that tree \cite{Ferrer2013e}. Those results still hold when $g(d)$ is not the identity function but a strictly monotonically increasing function of $d$ \cite{Ferrer2013e}. Furthermore, the placement of the hub in one extreme implies the maximum $D$ that a non-crossing tree (not necessarily a star) can achieve, which is $D=n(n-1)/2$ \cite{Ferrer2013b}.  
   \item
   A linear tree can be arranged linearly with $C=0$ and $D = D_\text{min} = n - 1$ (as in Fig. \ref{kinds_of_trees_figure}, which has no crossings and coincides with the smallest $D$ than an unrestricted tree can achieve (as $d_i \geq 1$ and a tree has $n-1$ edges). In contrast, a linear arrangement of the kind of Fig. \ref{maximal_non_crossing_linear_tree_figure} has $C= 0$ but yields $D=n(n-1)/2$, i.e. the maximum value of $D$ that a non-crossing tree can achieve \cite{Ferrer2013b}.
\end{itemize}
\item
$D = D_\text{min}$ does not imply $C=0$ in general. It has been shown that a linear arrangement of vertices with crossings can achieve a smaller value of $D$ than that of a minimum linear arrangement that minimizes $D$ when no crossings are allowed \cite{Hochberg2003a} (Fig. \ref{minimum_linear_arrangements_figure}). 
\end{itemize}
Thus, there is not a clear relationship between the minima of $D$ and $C$ when one abstracts from the structural properties of real syntactic dependency trees. The impact of the real properties of dependency structures for the arguments should be investigated in the future.

\begin{figure}
\includegraphics[scale = 0.8]{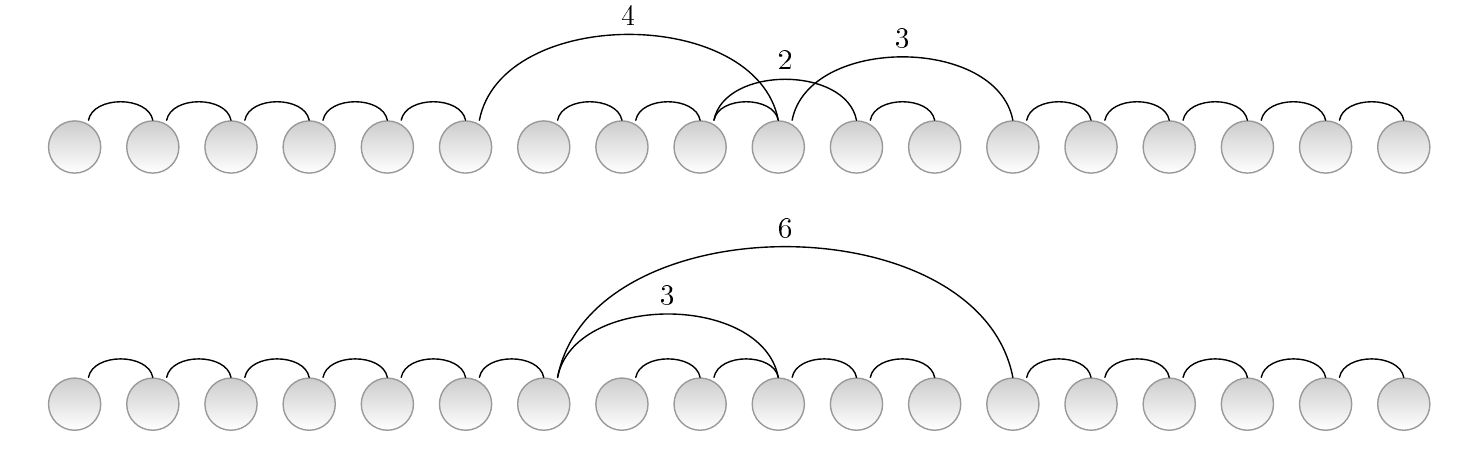}
\caption{
\label{minimum_linear_arrangements_figure}
Minimum linear arrangements of the same tree (only the length of edges that are longer than unity is indicated). 
{\bf Top} a minimum linear arrangement of a tree. The total sum of dependency lengths is $D = 4 + 2 + 3 + 14 = 23$. {\bf Bottom} a minimum linear arrangement of the same tree when crossings are disallowed. The total sum of dependency lengths is $D = 6 + 3 + 15 = 24$. Adapted from \cite{Hochberg2003a}. 
}
\end{figure}

As for the statistical relationship between $C$ and $D$, statistical analyzes support the hypothesis of a positive correlation between both at least in the domain between $n-1$, the minimum value of $D$, and $D=E_0[D]$ \cite{Ferrer2006d,Ferrer2013d}. For instance, crossings practically disappear if the vertices of a random tree are ordered to minimize $D$. The relationship between $C$ and $D$ in random permutations of vertices of the dependency trees is illustrated in Fig. \ref{correlation_figure}: $C$ tends to increase as $D$ increases from $D=D_\text{min}$ onwards. Results obtained with similar metrics \cite{Liu2008a,Liu2007a} are consistent with such a correlation. For instance, a measure of dependency length reduces in random trees when crossings are disallowed \cite{Liu2007a}. 
 
\begin{figure}
\includegraphics[scale = 0.6]{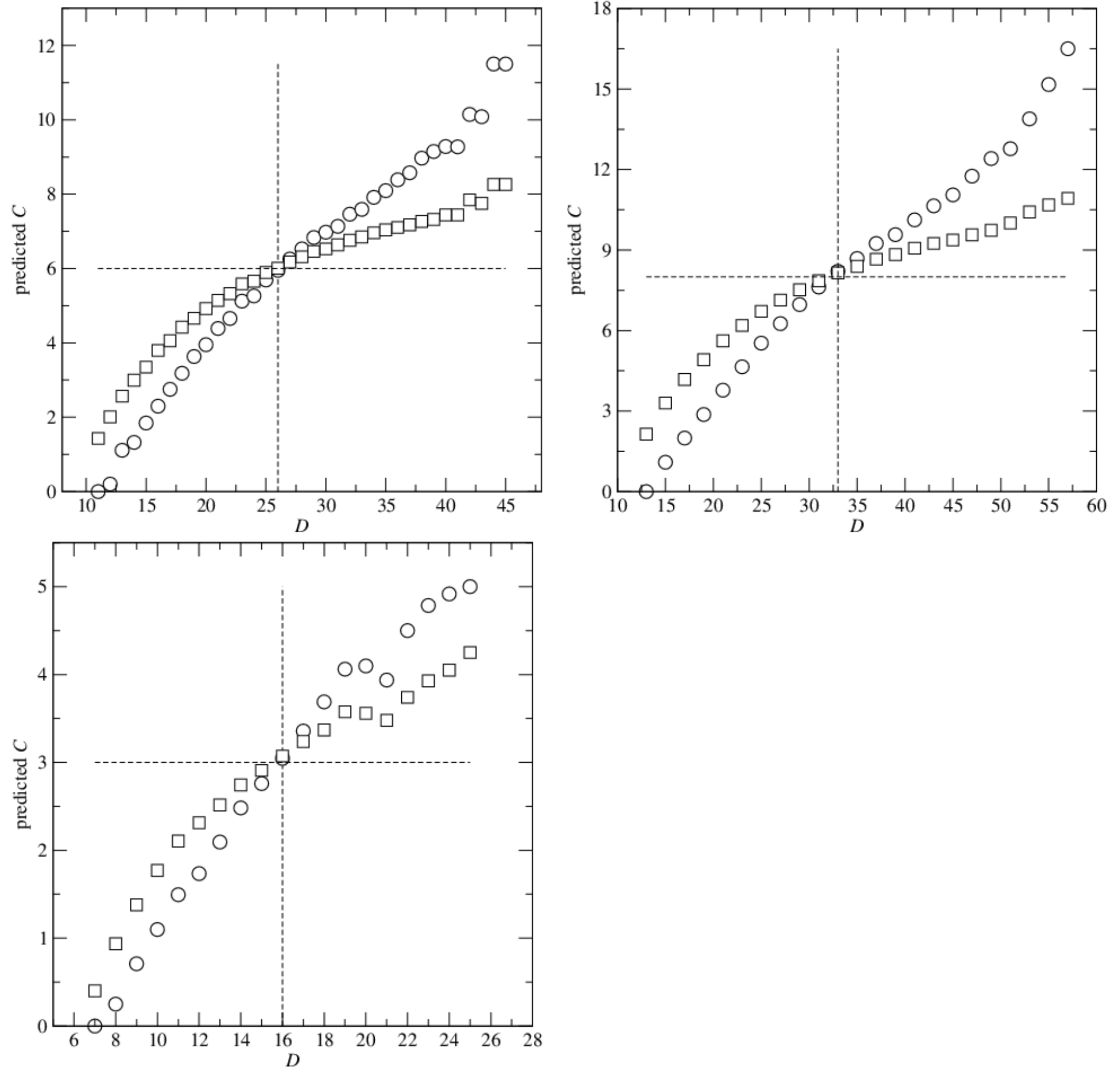}
\caption{
\label{correlation_figure}
Predictions about number of dependency crossings ($C$) as a function of the sum of dependency lengths ($D$) for dependency trees of real sentences.
$E[C|D]$, the average $C$ in all the possible permutations with a given value of $D$ 
 (\emph{circles}), is compared against the average $E_1[C]$ in those permutations (\emph{squares}). $E_1[C]$ is a prediction about $C$ based on information on the distance of just one of the vertices potentially involved in a crossing.
The \emph{vertical dashed line} indicates the value of $E_0[D]$ and the \emph{horizontal line} indicates the value of $E_0[C]$ (given a tree, those values are easy to compute with the help of Eqs.~\ref{expected_length_equation} and \ref{expected_crossings_equation}, respectively).
The value of $E[C|D]$ and its prediction is only shown for values of $D$ achieved by at least one permutation because certain values of $D$ cannot be achieved by a given tree. For a given tree, $D_\text{min}$ and $D_\text{max}$ are, respectively, the minimum and the maximum value of $D$ that can be reached ($D_\text{min} \leq D \leq D_\text{max}$).
{\bf Top-Left} sentence on top Fig. \ref{real_sentences_figure} with $D_\text{min} = 11$, $D_\text{max} = 45$, $E_0[D] = 26$ and $E_0[C] = 6$. 
{\bf Top-Right} sentence at the bottom of Fig. \ref{real_sentences_figure} with $D_\text{min} = 13$, $D_\text{max} = 57$, $E_0[D] = 33$ and $E_0[C] = 8$. Even values of $D$ are not found. 
{\bf Bottom-Left} results for the two sentences in Fig. \ref{extraposition_figure} with $D_\text{min} = 7$, $D_\text{max} = 25$, $E_0[D] = 16$ and $E_0[C] = 3$. Notice that the results are valid for  the couple of sentences in Fig. \ref{extraposition_figure} because they have the same structure in a different order. For this reason it is not surprising that $D_\text{min}$, $D_\text{max}$, $E_0[D]$, $E_0[C]$ and $E[C|D]$ coincide. Interestingly, $E_1[C]$ turns out to be the same for both, too. 
}
\end{figure}

This opens the problem of causality, namely if the minimization of $D$ may cause a minimization of $C$, or a minimization of $C$ may cause a minimization of $D$, or both principles cannot be disentangled or simply both principles are epiphenomena (correlation does not imply causality). Solving the problem of causality is beyond the scope of this article but we can however attempt to determine rationally which of the two forces, minimization of $D$ or minimization of $C$, might be the primary force by means of qualitative version of information theoretic model selection \cite{Burnham2002a}. The apparent tie between these two principles will be broken by the more limited explanatory power of the minimization of $C$.
The point is focusing on the phenomena that a principle of minimization of $C$ cannot illuminate. 
A challenge for that principle is the ordering of subject (S), object (O) and verb (V). The dependency structure of that triple is a star tree with the verb as the hub \cite{Ferrer2013e}. A tree of less than four vertices cannot have crossings \cite{Ferrer2013b}. Thus, $C=0$ regardless of the ordering of the elements of the triple. Interestingly, the principle of minimization of $C$ cannot explain why languages abandon SOV in favor of SVO \cite{Gell-Mann2011a}. In contrast, the attraction of the verb towards the center combined with the structure of the word order permutation space can explain it \cite{Ferrer2013e}. Another challenge for a principle of the minimization of $C$ are the relative ordering of dependents of nominal heads in SVO orders, that have been argued to preclude regression to SOV \cite{Ferrer2013e}. To sum up, a single principle of minimization of $C$ would compromise explanatory power and if its limitations were complemented with an additional principle of dependency length minimization then parsimony would be compromised. 

The reminder of the article is aimed at investigating a more parsimonious explanation for the ubiquity of non-crossing dependencies based on the minimization of $D$ as the primary force \cite{Ferrer2006d,Liu2008a,Morrill2009a,Ferrer2013b}. The minimization of $D$ would be a consequence of the minimization of cognitive effort: longer dependencies are cognitively more expensive \cite{Liu2008a,Morrill2000a,Gibson2000,Hawkins1994}. We will investigate two null hypotheses that allow one to predict the number of crossings as function of dependency lengths, which are determined by cognitive pressures. 

\section{A Stronger Null Hypothesis}

\label{stronger_null_hypothesis_section}

Here we consider a predictor for the number of crossings when some information about the length of the arcs is known.
The predictor guesses that number by considering, for every pair of edges that may potentially cross (pairs of edges sharing vertices cannot cross), the probability that they cross knowing the length of one of the edges and assuming that the vertices of the other edge have been arranged linearly at random. The null hypothesis in Sect. \ref{null_hypothesis_section} predicts the number of crossings in the same fashion but replacing that probability by the probability that two edges cross when both are arranged linearly at random (not arc length is given).
 
The null hypothesis in Sect. \ref{null_hypothesis_section} and the null hypothesis that will be explored in the current section, are reminiscent of two null hypotheses that are used in networks theory: random binomial graphs and random graphs with an arbitrary degree sequence or degree distribution \cite{Molloy1995a,Molloy1998a,Newman2001d} (in our case, information about dependency length plays an equivalent role to vertex degree in those models).

\subsection{The Probability that Two Edges Cross}

Let $\pi(v)$ be the position of the vertex $v$ in the linear arrangement ($\pi(v) = 1$ if $v$ is the 1st vertex of the sequence, $\pi(v) = 2$ it is the 2nd vertex and so on). $u\sim v$ is used to indicate an edge between vertices $u$ and $v$, namely that $u$ and $v$ are connected.
A vertex position $q$ is covered by the edge $u\sim v$ if and only if $min(\pi(u),\pi(v)) < q < max(\pi(u),\pi(v))$. A position $q$ is external to the edge $u \sim v$ if and only if $q < min(\pi(u),\pi(v))$ or $q > max(\pi(u),\pi(v))$. $s \sim t$ crosses $u \sim v$ if and only if 
\begin{itemize}
\item
either $\pi(s)$ is covered by $u \sim v$ and $\pi(t)$ is external to $u \sim v$ 
\item
or $\pi(t)$ is covered by $u \sim v$ and $\pi(s)$ is external to $u \sim v$. 
\end{itemize}
Notice that edges that share vertices cannot cross.

Let us consider first that no information is known about the length of two arcs. The probability that two edges, $s \sim t$ and $u\sim v$, cross when arranged linearly at random is $p(\text{cross}) = 1/3$ if the edges do not share any vertex and $p(\text{cross}) = 0$ otherwise \cite{Ferrer2013d}. We will investigate 
$p(\text{cross}|d)$, the probability that two edges, $s \sim t$ and $u\sim v$, cross when arranged linearly at random knowing that (a) one of the edges has length $d$, e.g., $|\pi(u)-\pi(v)| = d$ and (b) the edges do not share any vertex.

If $s\sim t$ and $u\sim v$ share a vertex, then $p(\text{cross}|d) = 0$. If not, 

\begin{proposition}
\begin{eqnarray}
p(\text{cross}|d) & = & 2\frac{(d-1)(n-d-1)}{(n-2)(n-3)} \nonumber \\
       & = & \frac{2(-d^2+ nd -n + 1)}{(n-2)(n-3)}. \label{probability_of_crossing_equation}
\end{eqnarray}
\end{proposition}
\begin{proof}
To see this notice that $d-1$ is the number of vertex positions covered by the edge of length $d$ and $n-d-1$ is the number of vertices that are external to that edge. Once the edge of length $d$ has been arranged linearly, there are 
\begin{equation}
{n-2 \choose 2} = \frac{(n-2)(n-3)}{2}
\end{equation}
possible placements for the two vertices of the other edge of which only $(d-1)(n-d-1)$ involve a position covered by the edge of length $d$ and another one that is external to that edge.
\qed
\end{proof}

$p(\text{cross}|d)$ and $p(\text{cross})=1/3$ are related, i.e. 
\begin{equation}
\sum_{d=1}^{n-1} p(\text{cross}|d) p(d) = \frac{1}{3}, \label{relationship_between_probabilities_of_crossing_equation}
\end{equation}
where 
\begin{equation}
p(d) = \frac{2(n - d)}{n(n-1)}
\label{probability_of_length_equation}
\end{equation}
is the probability that the linear arrangement of the two vertices of an edge yields a dependency of length $d$ \cite{Ferrer2004b}. Eq.~\ref{relationship_between_probabilities_of_crossing_equation} is easy to prove applying the definition of conditional probability ($p(\text{cross}|d)=p(\text{cross}, d)/p(d)$), which gives
\begin{equation}
\sum_{d=1}^{n-1} p(\text{cross}|d) p(d) = \sum_{d=1}^{n-1} p(\text{cross}, d) = p(\text{cross}) = \frac{1}{3}.
\end{equation}

When $n$ takes the smallest value needed for the possibility of crossings, i.e. $n = 4$ \cite{Ferrer2013d}, Eq.~\ref{probability_of_crossing_equation} yields $p(\text{cross}|1) = p(\text{cross}|3) = 0$ and $p(\text{cross}|2)=1$.
It is easy to show that 
\begin{itemize}
\item
$p(\text{cross}|d)$ is symmetric, i.e. $p(\text{cross}|d) = p(\text{cross}|n-d)$, 
\item
$p(\text{cross}|d)$ has two minima ($p(\text{cross}|d)=0$), at $d = 1$ and $d = n - 1$.
\item
\begin{equation}
p(\text{cross}|d) \leq p_\text{max}(\text{cross}|d),
\label{upper_bound_of_probability_of_crossing_equation}
\end{equation}
where 
\begin{equation}
p_\text{max}(\text{cross}|d) = \frac{\frac{n^2}{2} - 2(n-1)}{(n-2)(n-3)}.
\label{upper_bound_of_probability_of_crossing_equation2}
\end{equation}
To see this notice that $p(\text{cross}|d)$ is a function (Eq.~\ref{probability_of_crossing_equation}) that has a maximum at $d=d^*=n/2$. 
Applying $d=d^*$, Eq.~\ref{probability_of_crossing_equation} gives Eq.~\ref{upper_bound_of_probability_of_crossing_equation2}.
As $p(\text{cross}|d)$ is not defined for non-integer values of $d$, equality in Eq.~\ref{upper_bound_of_probability_of_crossing_equation} needs that $d^*$ is integer, namely that $n$ is even.
In the limit of large $n$, one has that $p_\text{max}(\text{cross}|d) = 1/2$.
\item
Accordingly, $p(\text{cross}|d)$ has either a maximum at $d = n/2$ if $n$ is even or two maxima, at $d = \lfloor n/2 \rfloor$ and $d = \lceil n/2 \rceil$ when $n$ is even because $d$ is a natural number.
\end{itemize}

\subsection{The Expected Number of Edge Crossings}

\label{expected_crossings_subsection}

Imagine that the structure of the tree is defined by an adjacency matrix $A = \{a_{uv} \}$ such that $a_{uv} = 1$ if the vertices $u$ and $v$ are linked and $a_{uv} = 0$ otherwise. Let $C$ be the number of edge crossings and $C(u,v)$ be the number of crossings where the edge formed by $u$ and $v$ is involved ($C(u,v) = 0$ if $u$ and $v$ are unlinked), i.e. 
\begin{equation}
C = \frac{1}{4} \sum_{u=1}^n \sum_{v=1}^n a_{uv} C(u,v)
\end{equation}
and 
\begin{equation}
C(u,v) = \frac{1}{2} \sum_{s=1, s \neq u,v}^n ~ \sum_{t=1, t \neq u,v}^n a_{st} C(u,v; s,t),
\end{equation}
where $C(u,v; s,t)$ indicates if $u,v$ and $s,t$ define a couple of edges that cross ($C(u,v; s,t) = 1$ if they cross, $C(u,v; s,t) = 0$ otherwise).
Thus, the expectation of $C$ is
\begin{equation}
E[C] = \frac{1}{4} \sum_{u=1}^n \sum_{v=1}^n a_{uv} E[C(u,v)].
\label{raw_expected_crossings_equation}
\end{equation}
In turn, the expectation of $C(u,v)$ is 
\begin{equation}
E[C(u,v)] = \frac{1}{2} \sum_{s=1, s \neq u,v}^n ~ \sum_{t=1, t \neq u,v}^n a_{st} E[C(u,v; s,t)].
\end{equation}
As $C(u,v; s,t)$ is an indicator variable,
\begin{equation}
E[C(u,v; s,t)] = p(s \sim t ~\&~ u \sim v~\text{cross}),
\end{equation}
namely $E[C(u,v; s,t)]$ is the probability that the edges $s \sim t$ and $u \sim v$ cross knowing that they do not share any vertex.

The number of edges that can cross the edge $u\sim v$ is $n-k_u-k_v$ (edges that share a vertex with $u \sim v$ cannot cross), which gives \cite{Ferrer2013d}
\begin{equation}
E[C(u,v)] = (n-k_u-k_v) p(\text{cross}| u \sim v) 
\end{equation}
assuming that the probability that the edge $u \sim v$ crosses another edge depends at most on $u \sim v$.
Applying the last result to Eq.~\ref{raw_expected_crossings_equation},
we obtain a general predictor of the number of crossings under a null of hypothesis $x$, i.e.
\begin{equation}
E_x[C] = \frac{1}{4} \sum_{u=1}^n \sum_{v=1}^n a_{uv} (n-k_u-k_v) p_x(\text{cross}| u \sim v),
\label{general_expected_crossings_equation}
\end{equation}
where $x$ indicates if the identity of one of the edges potentially involved a crossing, i.e. $u \sim v$ is actually known ($x = 1$ if it is known; $x= 0$ otherwise). 
Eq.~\ref{general_expected_crossings_equation} defines a family of predictors where $x$ is the parameter. 
$x=0$ corresponds to the null hypothesis in Sect. \ref{null_hypothesis_section}. To see it, notice that 
$x=0$, i.e.
\begin{equation}
p_x(\text{cross}| u \sim v) = p(\text{cross}) = 1/3
\end{equation}
transforms Eq.~\ref{general_expected_crossings_equation} into Eq.~\ref{expected_crossings_equation} \cite{Ferrer2013d}.
$E_x[C]$ can be seen as a simple approximation to the expected number of crossings in a linear random arrangement of vertices when all edge lengths are given. While $E_0[C]$ is a true expectation, $E_1[C]$ is not for not conditioning on the same lengths for every pair of edges that may cross.   

A potentially more accurate prediction of $C$ with regard to Eq.~\ref{expected_crossings_equation} is obtained when $x = 1$.  
For simplicity, let us reduce the knowledge of an edge to the knowledge of its length.
Let us define $d_{uv}= |\pi(u) - \pi(v)|$ as the distance between the vertices $u$ and $v$. If $x = 1$,  
the substitution of $p_x(\text{cross} | u \sim v)$ by $p(\text{cross}|d_{uv})$ in Eq.~\ref{general_expected_crossings_equation} yields $E_1[C]$, a prediction of $C$ where, for ever pair of edges that many potentially cross, the length of one edge is given and a random placement is assumed for the other edge, i.e. 
\begin{equation}
E_1[C] = \frac{1}{4} \sum_{u=1}^n \sum_{v=1}^n a_{uv} (n-k_u-k_v) p(\text{cross}|d_{uv}). \label{expected_crossings_based_on_distance_equation}
\end{equation}
Fig. \ref{correlation_figure} shows that $E_1[C]$ is positively correlated with $D$ on permutations of the vertices of the trees of a real sentences. Interestingly, the values of $E_1[C]$ overestimate the average $C$ when $D < E_0[D]$ and underestimate it when $D > E_0[D]$.

Variations in dependency lengths can alter $E_1[C]$. A drop in $p(\text{cross}|d_{uv})$ leads to a drop in $E_1[C]$. It has been shown above that 
$p(\text{cross}|d_{uv})$ is minimized by $d = 1$ and $d = n - 1$ and maximized by $d_{uv} \approx n/2$. 
When $d_{uv} < n/2$, a decrease in $d_{uv}$ decreases $p(\text{cross}|d_{uv})$. In contrast, a decrease in $d_{uv}$ when $d_{uv} > n/2$, increases $p(\text{cross}|d_{uv})$. 
However, the shortening of edges is more likely to decrease $p(\text{cross}|d)$ because
\begin{itemize}
\item
Tentatively, the linear arrangement of a tree can only have $n - d$ edges of length $d$ \cite{Ferrer2004b}.
\item
Under the null hypothesis that edges are arranged linearly at random, short edges are more likely than long edges. In that case, $p(d)$, the probability that an edge has length $d$ satisfies $p(d) \propto n - d$ (Eq. \ref{probability_of_length_equation}). 
\item
The potential ordering of a sentence is unlikely to have many dependencies of length greater than about $n/2$ because
the cost of a dependency is positively correlated with its length \cite{Morrill2000a,Gibson2004a}. 
\item
From an evolutionary perspective, initial states are unlikely to  involve many long dependencies \cite{Ferrer2014a}. 
\end{itemize}
Thus, it is unlikely that the shortening of edges increases $E_1[C]$. Instead, this is likely to decrease $E_1[C]$. Fig. \ref{correlation_figure} shows that
\begin{itemize}
\item
the average $E_1[C]$ is tends to increase as $D$ increases 
\item
the mean $C$ is bounded above by $E_1[C]$ (in the domain $D \leq E_0[D]$)
\end{itemize}
in concrete sentences. 

Applying the definition of $p(\text{cross}|d)$, given in Eq.~\ref{probability_of_crossing_equation}, to Eq.~\ref{expected_crossings_based_on_distance_equation}, yields
\begin{equation}
E_1[C] = \frac{B_2 - B_1}{(n-2)(n-3)} ,
\end{equation}
where 
\begin{equation}
B_2 = \frac{1}{2}\sum_{u=1}^n \sum_{v=1}^n a_{uv} (n-k_u-k_v)(n-d_{uv})d_{uv}\\
\label{contribution_of_dependency_lengths_equation}
\end{equation}
and 
\begin{align}
B_1 & = \frac{n-1}{2}\sum_{u=1}^n \sum_{v=1}^n a_{uv}(n-k_u-k_v) \nonumber \\
  & = \frac{n-1}{2}\left( n \sum_{u=1}^n \sum_{v=1}^n a_{uv} - 2 \sum_{u=1}^n k_{uv} \sum_{v=1}^n a_{uv} \right) \nonumber \\
  & = \frac{n-1}{2}\left( n \sum_{u=1}^n k_{u} - 2 \sum_{u=1}^n k^2_{u} \right) \nonumber \\
  & = \frac{n-1}{2}\left(2n(n-1) - 2 n \left<k^2\right> \right) && \text{applying Eqs.~\ref{sum_of_degrees_equation} and \ref{degree_2nd_moment_equation}} \nonumber \\
  & = n(n-1)\left(n-1 - \left<k^2\right> \right).
\end{align}
On the one hand, notice that $B_1\geq 0$ because $n \geq 1$ and $\left< k^2 \right> \leq n - 1$ \cite{Ferrer2013b}. On the other hand, notice that $B_2 \geq 0$ because 
\begin{itemize}
\item
The fact that $C(u,v) < 0$ is impossible by definition and also the fact that $C(u,v) \leq n - k_u - k_v$ \cite{Ferrer2013b} yields $n - k_u - k_v  \geq 0$.
\item
$d_{uv} \leq n - 1$ by definition \cite{Ferrer2013b}.
\end{itemize} 
Thus, $E_1[C]$ is proportional to the difference between two terms: a term $B_2$ that depends on dependency lengths and a term $B_1$ that depends exclusively on vertex degrees and sentence length (in words). Interestingly, Eq.~\ref{expected_crossings_equation} allows one to see $E_1[C]$ as a function of $E_0[C]$ since $B_1 = 6(n - 1) E_0[C]$. 

\section{Another Stronger Null Hypothesis}

\label{another_stronger_null_hypothesis_section}

Another prediction from a stronger null hypothesis is $E[C|D]$, the expected value of $C$ when one of the orderings (permutations) preserving the actual value of $D$ is chosen uniformly at random.
Fig. \ref{correlation_figure} shows 
that knowing the exact value of $D$, a positive correlation between $E[C|D]$ and $D$ follows for a concrete sentence (this is what circles are showing) and, interestingly, $E[C|D]$ predicts a smaller number of crossings than the average $E_1[C]$ as a function of $D$.

\section{Predictions, Testing and Selection}

\label{prediction_and_testing_section}

Table \ref{test_and_prediction_table} compares $E_0[C]$, $E_1[C]$ and $E[C|D]$ for the example sentences that have appeared so far. Table \ref{test_and_prediction_table} shows that, according to the normalized error, 
\begin{itemize}
\item
$E_0[C]$ makes the worst predictions.
\item
The predictions of $E[C|D]$ are the best except in one sentence where $E_1[C]$ wins.   
\end{itemize}
The take home message here is that it is possible to make quantitative predictions about the number of crossings, an aspect that theoretical approaches to the problem of crossing dependencies have missed \cite{Morrill2009a, deVries2012a, Levy2012a}. This quantitative requirement should not be regarded as a mathematical \emph{divertimento}: science is about predictions \cite{Bunge2013a}. 

\begin{table}
\caption{\label{test_and_prediction_table} The predicted crossings by each of the three null hypotheses, i.e. (1) random linear arrangement of vertices, (2) random linear arrangement with knowledge of the length of one of the edges that can potentially cross, and (3) random linear arrangement of vertices at constant sum of dependency lengths, for the example sentences employed in this article. Results of p-value testing for the 3rd null hypothesis are also included. $n$ is the number of vertices of the tree, $\left< k^2 \right>$ is the degree second moment about zero, $D$ is the sum of dependency lengths, $C$ is the actual number of crossings, $C_\text{max}$ is the potential number of crossings, i.e. $C_\text{max} = n \left( n-1-\left< k^2 \right>\right)/2$ \cite{Ferrer2013b}. $E_0[C]$, $E_1[C]$ and $E[C|D]$ are the predicted number of crossings according to the 1st, 2nd and 3rd hypotheses, respectively. $\epsilon_0[C]$, $\epsilon_1[C]$ and $\epsilon[C|D]$ are the normalized error of the 1st, 2nd and 3rd hypotheses, respectively, i.e. $\epsilon_x[C] = |C - E_x[C]|/C_\text{max}$ and $\epsilon[C|D] = |C - E[C|D]|/C_\text{max}$. Left and right p-values are provided for two statistics, $C$ and $|C - E[C|D]|$ under the 3rd null hypothesis. $R$ is the number of permutations of the vertex sequence where $D$ coincides with the original value. } 
\begin{tabular}{p{3.7cm}p{1.5cm}p{1.8cm}p{1.5cm}p{1.8cm}}
\hline\noalign{\smallskip}
 & ~Fig. \ref{real_sentences_figure}, top~ & ~Fig. \ref{real_sentences_figure}, bottom~ & ~Fig. \ref{extraposition_figure}, top~ & ~Fig. \ref{extraposition_figure}, bottom~ \\ 
\noalign{\smallskip}\svhline\noalign{\smallskip}
$n$ & 9 & 10 & 7 & 7 \\
$\left< k^2 \right>$ & 4 & 4.2 & 3.4 & 3.4 \\
$D$ & 13 & 17 & 15 & 10 \\
$C$ & 0 & 1 & 0 & 1 \\
$C_\text{max}$ & 18 & 24 & 9 & 9 \\
$E_0[C]$ & 6 & 8 & 3 & 3 \\
$\epsilon_0[C]$ & 0.33 & 0.29 & 0.33 & 0.22 \\
$E_1[C]$ & 2.4 & 4.2 & 1.2 & 2.2 \\
$\epsilon_1[C]$ & 0.13 & 0.14 & 0.13 & 0.13 \\
$E[C|D]$ & 1.1 & 2 & 2.8 & 1.1 \\
$\epsilon[C|D]$ & 0.062 & 0.041 & 0.31 & 0.011 \\
Left p-value of $C$& 0.28 & 0.37 & 0.058 & 0.69 \\
Right p-value of $C$& 1 & 0.88 & 1 & 0.75 \\
Left p-value of $|C-E[C|D]|$& 0.94 & 0.54 & 0.97 & 0.43 \\
Right p-value of $|C-E[C|D]|$& 0.33 & 0.71 & 0.088 & 1 \\
$R$ & 288 & 6664 & 548 & 102 \\
\noalign{\smallskip}\hline\noalign{\smallskip}
\end{tabular}
\end{table}

Our exploration of some sentences suggest that $E[C|D]$ makes better predictions in general. Notwithstanding we cannot rush to claim that $E[C|D]$ is the best model among those that we have considered only for that reason. According to standard model selection, the best model is one with the best compromise between the quality of its fit (the error of its predictions) and parsimony (the number of parameters) \cite{Burnham2002a}. Applying a rigorous framework for model selection is beyond the scope of this article but we can examine the parsimony of each model qualitatively to shed light on the best model. Interestingly, the three predictors vary concerning the amount of information that suffices to make a prediction. A comparison of the minimal information that each needs to make a prediction would be a better approach but that is beyond the scope of the current article. The value of $E_0[C]$ can be computed knowing only $n$ and $\left< k^2 \right>$ (Eq.~\ref{expected_crossings_equation}). The value of $E_1[C]$ can be computed knowing only $n$, the length of every edge and the degrees of the nodes forming every edge (Eqs.~\ref{probability_of_crossing_equation} and \ref{expected_crossings_based_on_distance_equation}). 
The calculation of $E[C|D]$ is less demanding than that of $E_1[C]$ concerning edge lengths. i.e. the total sum of the their lengths suffices (the length of individual edges is irrelevant), but still employs some information about edges, e.g., the nodes forming every edge (Sect. \ref{another_stronger_null_hypothesis_section}).
Our preliminary results (Table \ref{test_and_prediction_table}) and our analysis of parsimony from the perspective of sufficient information suggests that $E[C|D]$ has a better trade-off between the quality of its predictions and parsimony with respect to $E_1[C]$.  Interestingly, none of the models has free parameters if the ordering of the vertices and the syntactic dependency structure of the sentence is known as we have assumed so far. 

It is possible to perform traditional p-value testing of our models. For simplicity we focus on the null hypothesis that is defined in Sect. \ref{another_stronger_null_hypothesis_section}, i.e. permutations of vertices where the sum of edge lengths coincides with the true value (Table \ref{test_and_prediction_table} shows the number of permutations of this kind for each sentence). We consider two statistics for test test. First, $C$, the actual number of crossings of a dependency tree. Then one can define the left p-value as the proportion of those permutations with a number of crossings at most as large as the true value. Similarly, one can define the right p-value as the proportion of those permutations with a number of crossings at least as large as the true value.
One has to choose a significance level $\alpha$. The significance level must be such that there can be  p-values bounded above by $\alpha$ a priori. Otherwise, one is condemned to make type II errors. 
The smallest possible p-value is $1/R$, where $R$ is the number of permutations yielding the original value of $D$ (there is at least one permutation giving the same value of the statistic, i.e. the one that coincides with the original linear arrangement). Thus, $\alpha$ must exceed $1/R_\text{min}$, being $R_\text{min}$ the smallest value of $R$ in Table \ref{test_and_prediction_table}. $R_\text{min} = 102$ yields $\alpha \geq 1/102 \approx 0.01$.       
Thus we can safely choose a significance level of $\alpha = 0.05$. Table \ref{test_and_prediction_table} shows that the left and right p-values are always below $\alpha$, suggesting that the real numbers of crossings are compatible with those of this null hypothesis. However, notice that the p-value is borderline for one sentence (Fig. \ref{extraposition_figure}, top). 
Second, we consider $|C-E[C|D]|$, the absolute value of the difference between the actual number of crossings and the expected value, as another statistic to perform p-value testing. Accordingly, we define left and right p-values for the latter statistic following the same rationale used for the p-values of $C$. Table \ref{test_and_prediction_table} shows $|C-E[C|D]|$ is neither significantly low nor significantly large, 
thus providing support for the hypothesis that the number of crossings is determined by global constraints on dependency lengths.
  
Our p-value tests should be seen as preliminary statistical attempts. Future research should involve more languages and large dependency treebanks. Moreover, information theoretic models selection offers a much powerful approach over traditional p-value testing \cite{Burnham2002a} and should be explored in the future.    

\section{Discussion}

\label{discussion_section}

In this article, the possibility that the low number of crossings in dependency trees is a mere consequence of chance has been considered. 
As the expected number of crossings (when edge lengths are not given) decreases as the degree 2nd moment increases, a high hubiness could lead to a small number of crossings by chance, when dependency lengths are unconstrained. However, it has been shown that the hubiness required to have a small number of crossings in that circumstance would imply star trees, which are problematic for at least two reasons: real syntactic dependency trees of a sufficient length are not star trees and the diversity of possible syntactic dependency structures would be seriously compromised (Sect. \ref{null_hypothesis_section}). One cannot exclude that hubiness has some role in decreasing the potential number of crossings in sentences \cite{Ferrer2013b,Liu2007a} but it cannot be the only reason. "Grammar" has been examined as an explanation for the rather low frequency of crossings and the more parsimonious hypothesis of the minimization of syntactic dependency lengths \cite{Ferrer2006d,Liu2008a,Morrill2009a,Ferrer2013b} has been revisited (Sect. \ref{alternative_hypotheses_section}). Stronger null hypotheses involving partial information about dependency lengths suggest that the shortening of the dependencies is likely to imply a reduction of crossings (recall empirical evidence in Fig. \ref{correlation_figure} and general mathematical arguments based on one of those stronger null hypotheses in Sect. \ref{expected_crossings_subsection}). 
Moreover, it has been shown that a null hypothesis incorporating global information about dependency lengths, the sum of dependency lengths (Sect. \ref{another_stronger_null_hypothesis_section}) allows one to make specially accurate predictions about the actual number of crossings. The error of those predictions is neither surprisingly low nor high (Sect. \ref{prediction_and_testing_section}).
Our findings provide support for the hypothesis that uncrossing dependencies could be a side-effect of dependency length minimization, a principle that derives from the limited computational resources of the human brain \cite{Liu2008a,Morrill2000a,Gibson2000,Hawkins1994}. A universal grammar, a faculty of language or a competence-plus limiting the number of crossings might not be necessary. Just a version of Zipf's least effort might suffice \cite{Zipf1949a}. 

Upon a superficial analysis of facts, it is tempting to conclude that crossings cause processing difficulties and thus should be reduced. That follows easily from the correlation between crossings and dependency lengths that has been found in real and artificial sentence structures (e.g., \cite{Ferrer2013d,Liu2007a}). However, the cognitive cost of crossings does not need to be a direct consequence of the crossing \cite{Liu2008a} but a side effect of the longer dependencies that crossings are likely to involve \cite{Ferrer2006d,Morrill2009a}. It has been argued that a single principle of minimization of dependency crossings would compromise seriously parsimony and explanatory power (Sect. \ref{alternative_hypotheses_section}).

Although Fig. \ref{correlation_figure} shows that solving the minimum linear arrangement problem yields zero crossings for concrete sentences ($E[C|D] = 0$ and thus $C = 0$ for $D = D_\text{min}$), it is important to bear in mind that dependency length minimization cannot promise to reduce crossings to zero in all cases: the minimum linear arrangement of a tree can involve crossings (Fig. \ref{minimum_linear_arrangements_figure}). Interestingly, this means that the presence of some crossings in a sentence does not contradict a priori pressure for dependency length minimization or pressure for efficiency. Recall also the examples of real English in Fig. \ref{extraposition_figure}, showing that an ordering without crossings can have a higher sum of dependency lengths than one with crossings. 
This is specially important for Dutch, where crossing structures abound and have been shown to be easier to process that parallel non-crossing structures in German \cite{Bach1986a}. We believe that our theoretical framework might help to illuminate experiments suggesting that orderings with crossings tax working memory less than orderings with nesting \cite{deVries2012a}. 

In this article, we have addressed the problem of non-crossing dependencies from a really theoretical perspective. The arguments are a priori valid for any language and any linguistic phenomenon. This is a totally different approach from the investigation of 
non-crossing dependencies in a given language with a specific phenomenon (e.g., extraposition of relative clauses in English as in \cite{Levy2012a}, see also \cite{Bach1986a} for other languages). With such a narrow focus, the development of a general theory of word order is difficult. Generativists have been criticized for not having developed a general theory of language but a theory of English \cite{Evans2009a}. If one takes seriously recent concerns about the limits of building a theory from a sample of languages \cite{Piantadosi2013a}, it follows that hypotheses about non-crossing dependencies that abstract from linguistic details like ours (see also \cite{deVries2012a}) should receive more attention in the future.

Although we believe that non-crossing dependencies are the main reason why crossings dependencies do not occur very often in languages, we do not believe that cognitive pressure for dependency length minimization is the only factor involved in word order phenomena. The maximization of predictability or the structure of word order permutation space are crucial ingredients for a non-reductionist theory of word order \cite{Ferrer2013e,Ferrer2014a}. Word order is a multiconstraint engineering problem \cite{Ferrer2014a}.

Tentatively, a deep theory of syntax does not imply grammar or a language faculty exclusively. A well-known example is the case of sentence acceptability that may derive in some cases from processing constraints \cite{Morrill2000a,Hawkins2004a}. 
Our findings suggest that grammar may not be an autonomous entity but a series of phenomena emerging from physical or psychological constraints. Grammar might simply be an epiphenomenon \cite{Hopper1998a}. This is a more parsimonious hypothesis than grammar as a conventionalization of processing constraints \cite{Hawkins2004a}. Grammar may require fewer parameters than commonly believed. This is consistent with the idea that it would be desirable that the quantitative constraints of competence-plus are replaced by a theory of processing complexity or that the content of the "plus" derives from memory and integration costs \cite{Hurford2012_Chapter3}. Being the "plus" part non-empty, the point is elucidating whether the "competence" part is indeed empty or at least lighter than commonly believed.

A deep theory of language cannot be circumscribed to language: a deep physical theory for the fall of inanimate objects is also valid for animate objects (with certain refinements). A deep theory of syntactic structures and their linear arrangement does not need to be valid only for human language but also for other natural systems producing and processing sequences and operating under limited resources. For these reasons, our results should be extended to non-tree structures to investigate crossings in RNA structures \cite{Chen2009a}.

\begin{acknowledgement}
The deep and comprehensive review of Steffen Eger has been crucial for the preparation of the final version of this article. We owe him the current proof of Proposition \ref{above_quasistar_proposition}.   
We are also grateful to D. Blasi, R. \v{C}ech, M. Christiansen, J. M. Fontana, S. Frank, E. Gibson, H. Liu and G. Morrill for helpful comments or discussions. All remaining errors are entirely ours. This work was supported by the grant \emph{Iniciaci\'o i reincorporaci\'o a la recerca} from the Universitat Polit\`ecnica de Catalunya and the grants BASMATI (TIN2011-27479-C04-03) and OpenMT-2 (TIN2009-14675-C03) from the Spanish Ministry of Science and Innovation.
\end{acknowledgement}
\section*{Appendix}
\addcontentsline{toc}{section}{Appendix}
%
%

\subsection*{Tree Reduction}

As any tree of at least two vertices has at least two leaves \cite[p. 11]{Bollobas1998}, a tree of $n+1$ vertices (with $n > 1$) yields a reduced tree of $n$ vertices by removing one of its leaves. Notice that such a reduction from a tree of $n+1$ to a tree of $n$ nodes will never produce a disconnected graph.

Consider that a tree has a leaf that is attached to a vertex of degree $k$. Then, the sum of squared degrees of a tree of $n+1$ vertices, i.e. $K_2(n+1)$, can be expressed as a function of $K_2(n)$, the sum of squared degrees of a reduced tree of $n$ vertices, i. e.
\begin{eqnarray} 
K_2(n+1) & = & K_2(n) + k^2 - (k-1)^2 + 1 \nonumber \\
         & = & K_2(n) + 2k \label{tree_reduction_equation}
\end{eqnarray}
with $n \geq 0$.

\subsection*{The Only Tree that Has Degree Second Moment Greater than that of a Quasi-star Tree is a Star Tree. }

A quasi-star tree is a tree with one vertex of degree $n - 2$, one vertex of degree 2 and the remainder of vertices of degree 1. As a tree must be connected, that tree needs $n > 2$. The sum of squared degrees of a quasi-star tree is
\begin{eqnarray}
K_2^\text{quasi}(n) = (n-2)^2 + 4 + n - 2 = n^2 - 3n + 6.
\label{quasi_star_equation}
\end{eqnarray}
The sum of squared degrees of a star tree is \cite{Ferrer2013b}
\begin{eqnarray}
K_2^\text{star}(n) = n(n - 1).
\label{star_equation}
\end{eqnarray}
$K_2^\text{star}(n)$ is an upper bound of $K_2^\text{quasi}(n)$, more precisely, 
\begin{proposition}
  \label{quasistar_versus_star_proposition}
  For all $n\ge 3$, 
  \begin{align}\label{eq:toprove0}
    K_2^{\text{quasi}}(n) \leq K_2^{\text{star}}(n)  
  \end{align}
  with equality if and only if $n = 3$.
\end{proposition}
\begin{proof}
Applying the definitions in Eqs.~\ref{quasi_star_equation} and \ref{star_equation} to Eq.~\ref{eq:toprove0} we obtain 
\begin{equation} 
n^2 - 3n + 6 \leq n(n-1), 
\end{equation}
that is $n \geq 3$. 
\qed  
\end{proof}
$K_2(n)$ is maximized by star trees \cite{Ferrer2013b}. Quasi-star trees yield the second largest possible value of $K_2(n)$, more precisely  
\begin{proposition}
  \label{above_quasistar_proposition}
  For all $n\ge 3$, it holds that
  \begin{align}\label{eq:toprove}
    K_2(n)>K_2^\text{quasi}(n) \implies
    K_2(n)=K_2^\text{star}(n) 
  \end{align}
  for any tree with $n$ vertices. 
\end{proposition}

\begin{proof}
  Denote the antecedent of the implication in Eq.~\ref{eq:toprove}
  by $L(n)$ and the consequent by $R(n)$. We show by induction that,
  for all $n\ge 3$, 
  $L(n)\implies R(n)$.

  For $n=3$, the fact that the only possible tree is both a star and a quasi-star tree implies that $L(n)$ is false. Thus, Eq.~\ref{eq:toprove} holds trivially. 

  Let $n>3$. For the induction step, assume that
  $L(n)\implies R(n)$, and 
  also assume that $L(n+1)$ holds. We must show that then also $R(n+1)$
  holds. 

  Consider an arbitrary tree $T_{n+1}$ with $n+1$ vertices and consider the
  tree $T_n$, on $n$ vertices, with 
  a leaf $l$ removed from 
  $T_{n+1}$. 

  If $L(n)$ holds, then, 
  by the induction hypothesis, $R(n)$ holds, i.e., $
  K_2(n)=K_2^{\text{star}}(n)$. A tree with $n$ vertices for
  which $R(n)$ holds must be a star tree \cite{Ferrer2013d}; thus, $T_{n}$ is a star
  tree. Then, the leaf 
  vertex $l$ is  
  \begin{itemize}
  \item
  either attached to the hub of the star tree, in which
  case the 
  resulting tree $T_{n+1}$ is also a star tree, so that
  $K_{2}(n+1)=K_2^{\text{star}}(n+1)$, i.e., $R(n+1)$, 
  holds. 
  \item 
  or attached to a leaf of
  $T_n$, in which case $T_{n+1}$ is a quasi-star tree, contradicting that
  $L(n+1)$ holds. 
  \end{itemize}
  Conversely, if $L(n)$ does not hold, then $K_2(n)\le
  K_2^{\text{quasi}}(n)$. Accordingly, a tree $T_{n+1}$ with a leaf of degree $k$ satisfies (for any $k$) 
  \begin{align}
    K_2(n+1) = K_2(n)+2k \le K_2^{\text{quasi}}(n)+2k = n^2-3n+6+2k,
    \label{upper_bound_equation}  
  \end{align}
  being $K_2(n)$ the sum of squared degrees of the reduced tree.
  Now, we have assumed that $L(n+1)$ holds, i.e., that
  \begin{align}
    K_2(n+1)>K_2^{\text{quasi}}(n+1)= n^2-n+4,
    \label{lower_bound_equation}   
  \end{align}
  thanks to Eq.~\ref{quasi_star_equation}.
  Combining Eqs.~\ref{upper_bound_equation} and \ref{lower_bound_equation} it is obtained 
  \begin{align*}
    n^2-n+4<K_2(n+1)\le n^2-3n+6+2k,
  \end{align*}
  which implies 
  \begin{align*}
    2n< 2+2k.
  \end{align*}
  But this would require that $k>n-1$, that is, $k=n$, 
  since
  the maximum degree of a vertex in a tree with $n+1$ vertices is
  $n$. In other words, $T_{n+1}$ 
  would be forced to have a vertex of degree $k=n$, whence $T_{n+1}$
  is a star tree, so that $T_n$ also is a star tree (removing a leaf
  from a star tree yields a star tree). But, this would contradict
  that $L(n)$ does not hold 
  since $K_2^{\text{star}}(n) > K_2^{\text{quasi}}(n)$ for
  $n>3$ (Proposition \ref{quasistar_versus_star_proposition}). 
  \qed 
\end{proof}

\bibliographystyle{spphys}
\bibliography{biblio}

\begin{thebibliography}{10}
\providecommand{\url}[1]{{#1}}
\providecommand{\urlprefix}{URL }
\expandafter\ifx\csname urlstyle\endcsname\relax
  \providecommand{\doi}[1]{DOI \discretionary{}{}{}#1}\else
  \providecommand{\doi}{DOI \discretionary{}{}{}\begingroup
  \urlstyle{rm}\Url}\fi

\bibitem{McDonald2005a}
R.~McDonald, F.~Pereira, K.~Ribarov, J.~Haji\v{c}, in \emph{Proceedings of the
  Conference on Human Language Technology and Empirical Methods in Natural
  Language Processing} (Association for Computational Linguistics, Stroudsburg,
  PA, USA, 2005), HLT '05, pp. 523--530.
\newblock \doi{10.3115/1220575.1220641}

\bibitem{Miller1963}
G.A. Miller, N.~Chomsky, in \emph{Handbook of Mathematical Psychology}, vol.~2,
  ed. by R.D. Luce, R.~Bush, E.~Galanter (Wiley, New York, 1963), pp. 419--491

\bibitem{Chomsky1965}
N.~Chomsky, \emph{Aspects of the Theory of Syntax} (MIT Press, Cambridge, MA,
  1965)

\bibitem{Jackendoff2002a}
R.~Jackendoff, \emph{Foundations of Language} (Oxford University Press, Oxford,
  1999)

\bibitem{Newmeyer2001a}
F.~Newmeyer, Journal of Linguistics \textbf{37}, 101 (2001).
\newblock \doi{10.1017/S0022226701008593}

\bibitem{Newmeyer2003a}
F.J. Newmeyer, Language \textbf{79}, 682 (2003).
\newblock \doi{10.1353/lan.2003.0260}

\bibitem{Hurford2012_Chapter3}
J.R. Hurford, \emph{Chapter 3. Syntax in the Light of Evolution} (Oxford
  University Press, Oxford, 2012), pp. 175--258

\bibitem{Hawkins2004a}
J.A. Hawkins, \emph{Efficiency and Complexity in Grammars} (Oxford University
  Press, Oxford, 2004)

\bibitem{Hopper1998a}
P.J. Hopper, in \emph{The New Psychology of Language: Cognitive and Functional
  Approaches to Language Structure}, ed. by M.~Tomasello (Lawrence Erlbaum,
  Mahwah, NJ, 1998), pp. 155--175

\bibitem{Koehler2005a}
R.~K{\"o}hler, in \emph{Quantitative Linguistik. Ein internationales Handbuch.
  Quantitative Linguistics: An International Handbook} (Walter de Gruyter,
  Berlin/New York, 2005), pp. 760--775

\bibitem{Manning2002a}
C.D. Manning, in \emph{Probabilistic Linguistics}, ed. by R.~Bod, J.~Hay,
  S.~Jannedy (MIT Press, Cambridge, MA, 2002), pp. 289--341

\bibitem{Christiansen1999a}
M.H. Christiansen, N.~Chater, Cognitive Science \textbf{23}(2), 157 (1999).
\newblock \doi{10.1207/s15516709cog2302_2}

\bibitem{Burnham2002a}
K.P. Burnham, D.R. Anderson, \emph{Model Selection and Multimodel Inference. A
  Practical Information-theoretic Approach}, 2nd edn. (Springer, New York,
  2002)

\bibitem{Chomsky1995}
N.~Chomsky, \emph{The Minimalist Program} (MIT Press, 1995)

\bibitem{Hauser2002}
M.D. Hauser, N.~Chomsky, W.T. Fitch, Science \textbf{298}, 1569 (2002).
\newblock \doi{10.1126/science.298.5598.1569}

\bibitem{Hudson2007a}
R.~Hudson, \emph{Language Networks. The New Word Grammar} (Oxford University
  Press, Oxford, 2007)

\bibitem{Frank2012a}
S.~Frank, R.~Bod, M.H. Christiansen, Proceedings of the Royal Society B:
  Biological Sciences \textbf{279}, 4522–4531 (2012).
\newblock \doi{10.1098/rspb.2012.1741}

\bibitem{Miller1968a}
G.A. Miller, in \emph{The Psycho-Biology of Language: an Introduction to
  Dynamic Psychology (by G. K. Zipf)} (MIT Press, Cambridge, MA, USA, 1968),
  pp. v--x

\bibitem{Niyogi1995a}
P.~Niyogi, R.C. Berwick, A.I. Memo No. 1530 / C.B.C.L. Paper No. 118  (1995)

\bibitem{Suzuki2004a}
R.~Suzuki, P.L. Tyack, J.~Buck, Anim. Behav. \textbf{69}, 9 (2005).
\newblock \doi{10.1016/j.anbehav.2004.08.004}

\bibitem{Lecerf1960a}
Y.~Lecerf, Rapport CETIS No. 4 pp. 1--24 (1960).
\newblock Euratom

\bibitem{Hays1964}
D.~Hays, Language \textbf{40}, 511 (1964)

\bibitem{Chen2009a}
W.Y.C. Chen, H.S.W. Han, C.M. Reidys, Proceedings of the National Academy of
  Sciences \textbf{106}(52), 22061 (2009).
\newblock \doi{10.1073/pnas.0907269106}

\bibitem{Ferrer2012d}
R.~{Ferrer-i-Cancho}, A.~Hern\'{a}ndez-Fern\'{a}ndez, D.~Lusseau,
  G.~Agoramoorthy, M.J. Hsu, S.~Semple, Cognitive Science \textbf{37}(8),
  1565–1578 (2013).
\newblock \doi{10.1111/cogs.12061}

\bibitem{Ferrer2004b}
R.~{Ferrer-i-Cancho}, Physical Review E \textbf{70}, 056135 (2004).
\newblock \doi{10.1103/PhysRevE.70.056135}

\bibitem{Ferrer2011c}
R.~{Ferrer-i-Cancho}, F.~{Moscoso del Prado Mart\'{i}n}, Journal of Statistical
  Mechanics p. L12002 (2011).
\newblock \doi{10.1088/1742-5468/2013/07/L07001}

\bibitem{Moscoso2013a}
F.~{Moscoso del Prado}, in \emph{Proceedings of the 35th Annual Conference of
  the Cognitive Science Society}, ed. by M.~Knauff, M.~Pauen, N.~Sebanz,
  I.~Wachsmuth (Cognitive Science Society, Austin, TX, 2013), pp. 1032--1037

\bibitem{Ferrer2009b}
R.~{Ferrer-i-Cancho}, B.~Elvev{\aa}g, PLoS ONE \textbf{5}(4), e9411 (2009).
\newblock \doi{10.1371/journal.pone.0009411}

\bibitem{Piantadosi2011a}
S.T. Piantadosi, H.~Tily, E.~Gibson, Proceedings of the National Academy of
  Sciences \textbf{108}(9), 3526 (2011)

\bibitem{Uriagereka1998}
J.~Uriagereka, \emph{Rhyme and Reason. An Introduction to Minimalist Syntax}
  (The MIT Press, Cambridge, Massachusetts, 1998)

\bibitem{Hudson1984}
R.~Hudson, \emph{Word Grammar} (Blackwell, Oxford, 1984)

\bibitem{Melcuk1988}
I.~Mel'\v{c}uk, \emph{Dependency Syntax: Theory and Practice} (State of New
  York University Press, Albany, 1988)

\bibitem{Bollobas1998}
B.~Bollob\'as, \emph{Modern Graph Theory}.
\newblock Graduate Texts in Mathematics (Springer, New York, 1998)

\bibitem{Ferrer2003a}
R.~{Ferrer i Cancho}, R.V. Sol\'e, in \emph{Statistical Mechanics of Complex
  Networks}, \emph{Lecture Notes in Physics}, vol. 625, ed. by
  R.~Pastor-Satorras, J.~Rub\'i, A.~D\'iaz-Guilera (Springer, Berlin, 2003),
  pp. 114--125.
\newblock \doi{10.1007/b12331}

\bibitem{Goldberg2003a}
A.E. Goldberg, Trends in Cognitive Sciences \textbf{7}(5), 219  (2003).
\newblock \doi{10.1016/S1364-6613(03)00080-9}

\bibitem{GomezRodriguez2014a}
C.~G\'omez-Rodr{\'i}guez, D.~Fern{\'a}ndez-Gonz{\'a}lez, V.M.D. Bilbao,
  Computational Intelligence  (2014).
\newblock \doi{10.1111/coin.12027}

\bibitem{Noy1998a}
M.~Noy, Discrete Mathematics \textbf{180}, 301 (1998).
\newblock \doi{10.1016/S0012-365X(97)00121-0}

\bibitem{Ferrer2013d}
R.~{Ferrer-i-Cancho}, http://arxiv.org/abs/1305.4561  (2013)

\bibitem{Ferrer2013b}
R.~{Ferrer-i-Cancho}, Glottometrics \textbf{25}, 1 (2013)

\bibitem{Aldous1990a}
D.~Aldous, SIAM J. Disc. Math. \textbf{3}, 450 (1990).
\newblock \doi{10.1137/0403039}

\bibitem{Broder1989a}
A.~Broder, in \emph{Symp. Foundations of Computer Sci., IEEE} (New York, 1989),
  pp. 442--447

\bibitem{Moon1970a}
J.~Moon, in \emph{Canadian Math. Cong.} (1970)

\bibitem{Liu2010a}
H.~Liu, Lingua \textbf{120}(6), 1567 (2010).
\newblock \doi{10.1016/j.lingua.2009.10.001}

\bibitem{OEIS_A000055}
N.J.A. Seoane, \emph{Number of Trees with $n$ Unlabeled Nodes} (2013).
\newblock Available at http://oeis.org/A000055

\bibitem{Sole2010a}
R.V. Sol\'e, Complexity \textbf{16}(1), 20 (2010).
\newblock \doi{10.1002/cplx.20326}

\bibitem{Ferrer2012f}
R.~{Ferrer-i-Cancho}, N.~Forns, A.~Hern\'andez-Fern\'andez, G.~Bel-Enguix,
  J.~Baixeries, Complexity \textbf{18}(3), 11 (2013).
\newblock \doi{10.1002/cplx.21429}

\bibitem{Liu2008a}
H.~Liu, Journal of Cognitive Science \textbf{9}, 159 (2008)

\bibitem{Ferrer2013g}
R.~{Ferrer-i-Cancho}, http://arxiv.org/abs/1310.5884  (2013)

\bibitem{Ferrer2006d}
R.~{Ferrer-i-Cancho}, Europhysics Letters \textbf{76}(6), 1228 (2006).
\newblock \doi{10.1209/epl/i2006-10406-0}

\bibitem{Morrill2009a}
G.~Morrill, O.~Valent{\'i}n, M.~Fadda, in \emph{Logic, Language, and
  Computation}, \emph{Lecture Notes in Computer Science}, vol. 5422, ed. by
  P.~Bosch, D.~Gabelaia, J.~Lang (Springer Berlin Heidelberg, 2009), pp.
  272--286.
\newblock \doi{10.1007/978-3-642-00665-4_22}

\bibitem{Ferrer2014a}
R.~{Ferrer-i-Cancho}, in \emph{{THE EVOLUTION OF LANGUAGE - Proceedings of the
  10th International Conference (EVOLANG10)}}, ed. by E.A. Cartmill,
  S.~Roberts, H.~Lyn, H.~Cornish (Wiley, Vienna, Austria, 2014), pp. 66--73.
\newblock \doi{10.1142/9789814603638_0007}.
\newblock Evolution of Language Conference (Evolang 2014), April 14-17

\bibitem{Ferrer2013e}
R.~{Ferrer-i-Cancho}, Language Dynamics and Change \textbf{4}, in press (2015).
\newblock \urlprefix\url{http://arxiv.org/abs/1309.1939}

\bibitem{Chung1984}
F.R.K. Chung, Comp. \& Maths. with Appls. \textbf{10}(1), 43 (1984).
\newblock \doi{10.1016/0898-1221(84)90085-3}

\bibitem{Ferrer2013c}
R.~{Ferrer-i-Cancho}, H.~Liu, Glottotheory \textbf{5}, 143 (2014).
\newblock \doi{10.1515/glot-2014-0014}

\bibitem{Morrill2000a}
G.~Morrill, Computational Linguistics \textbf{25}(3), 319 (2000).
\newblock \doi{10.1162/089120100561728}

\bibitem{Hawkins1998a}
J.A. Hawkins, in \emph{Constituent Order in the Languages of {Europe}}, ed. by
  A.~Siewierska (Mouton de Gruyter, Berlin, 1998)

\bibitem{Levy2012a}
R.~Levy, E.~Fedorenko, M.~Breen, E.~Gibson, Cognition \textbf{122}(1), 12
  (2012).
\newblock \doi{10.1016/j.cognition.2011.07.012}

\bibitem{Gibson2010a}
E.~Gibson, E.~Fedorenko, Trends in Cognitive Sciences \textbf{14}(6), 233
  (2010).
\newblock \doi{10.1016/j.tics.2010.03.005}

\bibitem{Culicover2010a}
P.W. Culicover, R.~Jackendoff, Trends in Cognitive Sciences \textbf{14}(6), 234
  (2010).
\newblock \doi{10.1016/j.tics.2010.03.012}

\bibitem{Baronchelli2013a}
A.~Baronchelli, R.~{Ferrer-i-Cancho}, R.~Pastor-Satorras, N.~Chatter,
  M.~Christiansen, Trends in Cognitive Sciences \textbf{17}, 348 (2013).
\newblock \doi{10.1016/j.tics.2013.04.010}

\bibitem{Hochberg2003a}
R.A. Hochberg, M.F. Stallmann, Information Processing Letters \textbf{87}, 59
  (2003).
\newblock \doi{10.1016/S0020-0190(03)00261-8}

\bibitem{Liu2007a}
H.~Liu, Glottometrics \textbf{15}, 1 (2007)

\bibitem{Gell-Mann2011a}
M.~Gell-Mann, M.~Ruhlen, Proceedings of the National Academy of Sciences USA
  \textbf{108}(42), 17290 (2011).
\newblock \doi{10.1073/pnas.1113716108}

\bibitem{Gibson2000}
E.~Gibson, in \emph{Image, Language, Brain} (The MIT Press, Cambridge, MA,
  2000), pp. 95--126

\bibitem{Hawkins1994}
J.A. Hawkins, \emph{A Performance Theory of Order and Constituency} (Cambridge
  University Press, New York, 1994)

\bibitem{Molloy1995a}
M.~Molloy, B.~Reed, Random Structures and Algorithms \textbf{6}, 161 (1995).
\newblock \doi{10.1002/rsa.3240060204}

\bibitem{Molloy1998a}
M.~Molloy, B.~Reed, Combinatorics, Probability and Computing \textbf{7}, 295
  (1998).
\newblock \doi{10.1017/S0963548398003526}

\bibitem{Newman2001d}
M.E.J. Newman, S.H. Strogatz, D.J. Watts, Phys. Rev. E \textbf{64}, 026118
  (2001).
\newblock \doi{10.1103/PhysRevE.64.026118}

\bibitem{Gibson2004a}
E.~Gibson, T.~Warren, Syntax \textbf{7}, 55 (2004).
\newblock \doi{10.1111/j.1368-0005.2004.00065.x}

\bibitem{deVries2012a}
M.H. de~Vries, K.M. Petersson, S.~Geukes, P.~Zwitserlood, M.H. Christiansen,
  Philosophical Transactions of the Royal Society B: Biological Sciences
  \textbf{367}(1598), 2065 (2012).
\newblock \doi{10.1098/rstb.2011.0414}

\bibitem{Bunge2013a}
M.~Bunge, \emph{La ciencia. Su m\'etodo y su filosof\'ia} (Laetoli, Pamplona,
  2013)

\bibitem{Zipf1949a}
G.K. Zipf, \emph{Human Behaviour and the Principle of Least Effort}
  (Addison-Wesley, Cambridge (MA), USA, 1949)

\bibitem{Bach1986a}
E.~Bach, C.~Brown, W.~Marslen-Wilson, Language and Cognitive Processes
  \textbf{1}, 249 (1986).
\newblock \doi{10.1080/01690968608404677}

\bibitem{Evans2009a}
N.~Evans, S.C. Levinson, Behavioral and Brain Sciences \textbf{32}, 429 (2009).
\newblock \doi{10.1017/S0140525X0999094X}

\bibitem{Piantadosi2013a}
S.~Piantadosi, E.~Gibson, Cognitive Science \textbf{38}(4), 736 (2014).
\newblock \doi{10.1111/cogs.12088}

\end{thebibliography}

\end{document}